\documentclass{article}

\usepackage{microtype}
\usepackage{graphicx}
\usepackage[T1]{fontenc}
\usepackage{pifont}

\usepackage{smile}

\usepackage{hyperref}

\usepackage[accepted]{icml2023}

\usepackage{amsmath}
\usepackage{amssymb}
\usepackage{mathtools}
\usepackage{amsthm}
\newcommand{\diff}{\mathrm{d}}
\allowdisplaybreaks

\usepackage[capitalize,noabbrev]{cleveref}

\usepackage[compact]{titlesec}
\usepackage{sidecap}
\titlespacing*{\section}{0pt}{*0.9}{*0.9}
\titlespacing*{\subsection}{0pt}{*0.9}{*0.9}
\titlespacing*{\subsubsection}{0pt}{*0.5}{*0.5}
\usepackage[textsize=tiny]{todonotes}

\icmltitlerunning{Cooperative Multi-Agent Reinforcement Learning: Asynchronous Communication and Linear Function Approximation}

\begin{document}

\twocolumn[
\icmltitle{Cooperative Multi-Agent Reinforcement Learning: \\Asynchronous Communication and Linear Function Approximation}



\icmlsetsymbol{equal}{*}

\begin{icmlauthorlist}
\icmlauthor{Yifei Min}{equal,yale}
\icmlauthor{Jiafan He}{equal,ucla}
\icmlauthor{Tianhao Wang}{equal,yale}
\icmlauthor{Quanquan Gu}{ucla}
\end{icmlauthorlist}

\icmlaffiliation{yale}{Department of Statistics and Data Science, Yale University}
\icmlaffiliation{ucla}{Department of Computer Science, University of California, Los Angeles}

\icmlcorrespondingauthor{Quanquan Gu}{qgu@cs.ucla.edu}

\icmlkeywords{Markov Decision Process, reinforcement learning, multi-agent, federated learning, distributed}

\vskip 0.3in
]



\printAffiliationsAndNotice{\icmlEqualContribution} % otherwise use the standard text.

\begin{abstract}
We study multi-agent reinforcement learning in the setting of episodic Markov decision processes, where multiple agents cooperate via communication through a central server.
We propose a provably efficient algorithm based on value iteration that enable asynchronous communication while ensuring the advantage of cooperation with low communication overhead.
With linear function approximation, we prove that our algorithm enjoys an 
$\tilde{\mathcal{O}}(d^{3/2}H^2\sqrt{K})$ regret  with 
$\tilde{\mathcal{O}}(dHM^2)$ communication complexity, where $d$ is the feature 
dimension, $H$ is the horizon length, 
$M$ is the total number of agents, and $K$ is the total number of episodes.
We also provide a lower bound showing that a minimal $\Omega(dM)$ communication 
complexity is required to improve the performance through collaboration.
\end{abstract}

\section{Introduction}
Multi-agent Reinforcement Learning (RL) has been successfully applied in various application scenarios, such as robotics~\citep{williams2016aggressive,liu2019lifelong,ding2020distributed,liu2020indoor}, games~\citep{vinyals2017starcraft,berner2019dota,jaderberg2019human,ye2020towards}, and many other real-world systems and settings
\citep{bazzan2009opportunities,yu2014multi,yu2020deep,fei2022cascaded,min2022learn,xu2023finding}.
In particular, in the cooperative setting, agents benefit from collaboration via (in)direct communication among each other.
% To efficiently exploit this advantage, 
It thus requires the RL algorithm to effectively coordinate the communication in a flexible way, in order to fully exploit the advantage of cooperation.
% , which poses unique challenges compared with the single-agent setting.
Towards this goal, in this paper, we study cooperative multi-agent RL with asynchronous communication, and show that the same  performance as single-agent methods can be achieved with efficient communication strategy.
% \highlight{[a high level summary of the question studied in our paper]}.

We focus on the so-called parallel RL setting
\citep{kretchmar2002parallel,grounds2007parallel}, where agents interact with the 
environment in parallel to solve a common problem.
More specifically, we consider a model of episodic Markov decision processes (MDPs) called \emph{linear MDPs}~\citep{yang2019sample,jin2020provably}, where both the transition probability and reward functions are linear in some known $d$-dimensional feature mapping.
We assume there are $M$ agents, which share the same underlying MDP model, but interact with the environment 
independently in parallel. %through a central server.
%The agents .
The agents cannot communicate directly with each other, and the information 
exchange is realized only through a central server.
We emphasize that in our setting, the communication between the agents and server 
is not required to be synchronous, and any communication is initiated solely by 
the agent, thus providing flexibility for practical needs.
The goal of the agents is to achieve a low total regret with as less communication as 
possible.

Notably, a recent work by \citet{dubey2021provably} studied cooperative multi-agent RL with linear MDPs. They proposed a cooperative variant of the \texttt{LSVI-UCB} 
algorithm~\citep{jin2020provably} named \texttt{Coop-LSVI}, which achieves an 
$\tilde O(d^{3/2}H^2\sqrt{K})$ regret\footnote{Their original result is written as $\tilde O(d^{3/2}H^2\sqrt{MT})$. Here $d$ is the feature 
dimension, $H$ is the horizon length, 
$M$ is the total number of agents, and $K$ is the total number of episodes. Because of their round-robin-type participation, their $MT$ is equivalent to $K$, which is the total number of episodes under our notation.} with $O(dHM^3)$ communication complexity.
However, their algorithm mandates the participation of all agents in a round-robin fashion,
which is impractical as it imposes a stringent synchronous constraint on the  agents' interaction with the environment and their communication with the server.
It is possible that some agents might be temporarily unavailable in a round, or the connection with the server is disrupted due to infrastructure failure. 
These anomalies demand the algorithm to be resilient to irregular participation 
patterns of the agents.

To this end, we propose an asynchronous version of 
\texttt{LSVI-UCB}~\citep{jin2020provably}.
We eliminate the synchronous constraint by carefully designing a determinant-based 
criterion for deciding whether or not an agent needs to communicate with the 
server and update the local model.
This criterion depends only on the local data of each agent, and more importantly, 
the communication triggered by one agent with the server \emph{does not} 
affect any other agents.
As a comparison, in the \texttt{Coop-LSVI} algorithm in \citet{dubey2021provably},
if some agents decide to communicate with the server, the algorithm will 
execute a mandated aggregation of data from all agents.
As a result, our algorithm is considerably more flexible and practical, though this presents new challenges in analyzing its performance theoretically.

\begin{table*}[t]
\centering
% \begin{tabular}{1\linewidth}{ccccccc}
\begin{tabular}{cccccc}
\toprule
Setting    & Algorithm     & Regret  & { Communication} & { Low-switching} & { $\substack{\text{Allow asynchronous}\\ \text{communication}}$} \\
\midrule
\multirow{2}{*}{{ Single-agent}} & \texttt{LSVI-UCB}   & \multirow{2}{*}{$d^{3/2}H^2\sqrt{K}$}  
& \multirow{2}{*}{N/A} & \multirow{2}{*}{\ding{56}} & \multirow{2}{*}{N/A}\\
& {\tiny\citep{jin2020provably}}& & & & \\
\hline
\multirow{2}{*}{ Multi-agent} & \texttt{Coop-LSVI}  & \multirow{2}{*}{$d^{3/2} H^2 \sqrt{K}$} 
& \multirow{2}{*}{$d H M^3$} & \multirow{2}{*}{\ding{51}} & \multirow{2}{*}{\ding{56}} \\ 
& {\tiny \citep{dubey2021provably}} & & & &\\ 
\hline
\multirow{2}{*}{ Multi-agent} & \texttt{Async-Coop-LSVI-UCB} 
& \multirow{2}{*}{$d^{3/2}H^2\sqrt{K} $} & \multirow{2}{*}{$d H M^2$} 
& \multirow{2}{*}{\ding{51}} & \multirow{2}{*}{ \ding{51}} \\
& { (ours)} & & & & \\
\bottomrule
\end{tabular}
\caption{Comparison of our result with baseline methods for linear MDPs. 
Our result achieves regret comparable to that of the single-agent setting under low communication complexity. 
Here $d$ is the dimension of the feature, $M$ is the number of agents, and 
$K$ is the total number of episodes by all agents. 
Logarithmic factors are hidden from the regret and the communication complexity.
}
\label{table: comparison with baseline}
\end{table*}

As mentioned before, the participation order of the agents can be arbitrary
and irregular, resulting in some agents having the latest aggregated information from the server while others may have outdated information.
This issue of \emph{information asymmetry} prohibits direct adaption of existing 
analyses for LSVI-type algorithms, such as those in \citet{jin2020provably,dubey2021provably}.
To address this, we need to carefully calibrate the communication criterion, so as to 
balance and regulate the information asymmetry.
We achieve this by examining the quantitative relationship between each agent's local 
information and the virtual universal information, yielding a simple yet effective 
communication coefficient.
The final result confirms the efficiency of the proposed algorithm (see 
Theorem~\ref{thm: upper bound}). 

Besides the positive result, we further investigate the fundamental limit of 
cooperative multi-agent RL.
Inspired by the construction of hard-to-learn instance for federated bandits in 
\citet{wang2019distributed, he2022simple}, we characterize the minimum amount 
of communication complexity required to surpass the performance of a single-agent method.\footnote{Here by `single-agent methods' we mean all agents independently 
run a single-agent algorithm without communication.}
(see Theorem~\ref{thm:lower_bound}).

The main contributions of this paper are summarized as follows:
\begin{itemize}
    \item We propose a provably efficient algorithm (Algorithm~\ref{alg: main}) 
    for cooperative multi-agent RL with asynchronous communication under episodic linear MDPs~\citep{yang2019sample,jin2020provably}. 
    Our algorithm allows an arbitrary participation order of the agents and 
    independent communication between the agents and the server, making it significantly more flexible than the existing algorithm in 
    \citet{dubey2021provably} for the synchronous setting.
    A comparison with baseline methods is presented in Table~\ref{table: comparison with baseline}.

    \item We prove that under standard assumptions, the proposed algorithm 
    enjoys an $\tilde{\mathcal{O}}(d^{3/2}H^2\sqrt{K})$ regret with 
    $\tilde{\mathcal{O}}(dHM^2)$ communication complexity. %where $d$ is the embedding 
    %dimension, $H$ is the horizon length, $M$ is the total number of agents, 
    %and $K$ is the total number of episodes.
    Our theoretical analysis identifies and resolves the information assymmetry 
    due to asynchronous communication, which may be of independent interest.

    \item We also provide a lower bound for the communication complexity, showing that 
    an $\Omega(dM)$ complexity is necessary to improve over single-agent 
    methods through collaboration.
    To the best of our knowledge, this is the first result on communication complexity for learning multi-agent MDPs.
\end{itemize}

\paragraph{Notation.}
We denote $[n] := \{1,2, \ldots, n\}$ for any positive integer $n$.
We use $\Ib$ to denote the $d\times d$ identity matrix. 
We use $\mathcal{O}$ to hide universal constants and $\tilde{\mathcal{O}}$ to further hide poly-logarithmic terms.
For any vector $\xb \in \RR^d$ and positive semi-definite matrix 
$\bSigma \in \RR^{d \times d}$, we denote $\|\xb\|_{\bSigma} = \sqrt{\xb^\top \bSigma \xb}$.
For any $a,b,c\in\RR\cup\{\pm\infty\}$ with $a\leq b$, we use the shorthand $[c]_{[a,b]}$ to denote the truncation (or projection) of $c$ into the interval $[a,b]$, i.e., $[c]_{[a,b]} = \argmin_{c'\in[a,b]} |c-c'|$.
A comprehensive clarification of notation is also provided in Appendix~\ref{sec: clarification of notation}.

\section{Related Work}

\paragraph{Multi-Agent RL.}
Various algorithms with convergence guarantees have been developed for multi-agent RL~\cite{zhang2018fully, wai2018multi, zhang2018networked}, e.g., federated version of TD and Q-learning \citep{khodadadian2022federated}, and policy gradient with fault tolerance~\citep{fan2021fault}.
In contrast, in this work we study algorithms with low regret guarantee cooperative multi-agent RL with asynchronous communication.

As mentioned above, we focus on the homogeneous setting where the underlying MDP for every agent is the same.
There are also existing works on cooperative multi-agent RL with non-stationary environment and/or heterogeneity~\citep{lowe2017multi,yu2021surprising,kuba2021trust,liu2022stateful,jin2022federated}.
Besides homogeneous parallel linear MDP, \citet{dubey2021provably} further studied 
heterogeneous parallel linear MDP (i.e., the underlying MDPs can be different 
from agent to agent) and Markov games in linear multi-agent MDPs.
These generalized setups are beyond the scope of the current paper, and we leave 
as future work to study algorithms compatible with asynchronous communication 
in these settings.

We consider multi-agent RL with linear function approximation to incorporate large state and action space.
More powerful deep learning techniques have been used for federated RL in~\citet{clemente2017efficient,espeholt2018impala,horgan2018distributed,nair2015massively,zhuo2019federated}.
We refer the reader to \citet{qi2021federated} for a recent survey on federated RL.
Our work is also related to the broader context of distributed learning, where a collective of agents collaborate towards a common objective~\citep{bottou2010large,dean2012large,littman2013distributed,li2014communication,liang2018rllib,hoffman2020acme,ding2022towards, zhan2021deepmtl,zhan2022deepmtlpro,xu2023achieving}. 
Interested readers may refer to the survey article by~\citet{verbraeken2020survey}.

\paragraph{RL with Linear Function Approximation.}
Function approximation techniques in RL enable extension beyond the restricted setting of tabular MDP.
Recent years have especially witnessed rapid progress in the research of single-agent RL with linear function approximation, among which two major parallel lines of work (for online RL) focus on linear MDPs~\citep{yang2019sample,jin2020provably,zanette2020learning,neu2020unifying, he2020logarithmic,wang2021provably,hu2022nearly,he2022nearly,agarwal2022vo,lu2022pessimism} and linear mixture MDPs~\citep{modi2020sample,jia2020model,ayoub2020model,zhou2020provably,cai2020provably,zhou2020nearly,zhang2021variance,kim2021improved,min2022learning,zhang2022horizon,zhou2022computationally}, respectively.

In this paper, we follow the design of the \texttt{LSVI-UCB} algorithm~\citep{jin2020provably} to devise an asynchronous algorithm for cooperative linear MDP.
Indeed, our algorithmic design can also be carried over to tabular MDPs and linear mixture MDPs, which will be discussed later in Section~\ref{sec:algorithm}.

 % \citet{yang2019sample} studied discounted linear MDPs with a generative model, and \citet{jin2020provably} proposed an efficient \texttt{LSVI-UCB} algorithm for linear MDPs without a generative model. It has been shown by \citet{du2019good} that MDP with misspecified linear function approximation could be exponentially hard to learn. Linear MDPs under various settings have also been studied by \citep{zanette2020learning,neu2020unifying, he2020logarithmic,wang2021provably}. 

% A parallel line of work studies linear mixture MDPs \citep{jia2020model,ayoub2020model,cai2020provably,zhou2020nearly, min2021learning} (a.k.a., linear kernel MDPs \citep{zhou2020provably}) where the transition kernel is a linear function of a ternary feature mapping $\psi:\cS\times\cA\times\cS\to\RR^d$. In particular, \citet{zhou2020nearly} achieved a nearly minimax regret bound by carefully utilizing the variance information of the value functions. \citet{zhang2021variance} constructed a variance-aware confidence set for time-homogeneous linear mixture MDPs.

\section{Preliminaries}

In this section, we first provide the formal definition of linear MDPs, and then 
introduce our model of cooperative multi-agent linear MDPs with asynchronous communication.

\subsection{Linear MDPs}
% \revise
{Episodic MDPs are a classic family of models in RL~\citep{sutton2018reinforcement}.
Let $\cS$ be the state space, $\cA$ be the action space, and $H$ be the horizon length.
Each episode starts from some initial state $s_0\in\cS$. For step $h=1,2,\ldots,H$, the agent at state $s_h\in\cS$ takes some action $a_h\in\cA$, and receives a reward $r_h(s_h,a_h)$, where $r_h:\cS\times\cA\to\RR$ is the reward function at step $h$.
Then the environment transits to the next state $s_{h+1} \sim \PP_h(\cdot\mid s_h,a_h)$, where $\PP_h:\cS\times\cS\times\cA\to\RR$ is the transition probability function for step $h$.
We call the strategy the agent interacts with the environment a \emph{policy}, and a policy $\pi$ consists of $H$ mappings, $\pi_h:\cS\to\cA$ for every $h\in[H]$.
The agent will run for $K$ episodes in total. 
The goal of the agent is then to find the optimal policy that maximizes the cumulative reward across an episode through this online process.}

In this work we consider the time-inhomogeneous linear MDP setting where the transition probabilities and the reward functions can be parametrized as linear functions of a known feature mapping $\bphi$. 
This is a popular setting considered by various authors \citep{bradtke1996linear, melo2007q, yang2019sample, jin2020provably, min2021variance, yin2021near}. 
The formal definition is given by the following assumption. 

\begin{assumption}[Linear MDPs, \citealt{yang2019sample,jin2020provably}]\label{assump: linear_MDP} 
$\text{MDP}(\cS, \cA, H, \{r_h\}_{h=1}^H, \{\PP_h\}_{h=1}^H)$ is called a linear MDP with a \emph{known} feature mapping $\bphi: \cS \times \cA \to \RR^d$, if for any $h \in [H]$, there exist $\bgamma_h$ and $\bmu_h \in \RR^d$, such that for any state-action pair $(s,a)\in \cS \times \cA$,
\begin{align}\label{eq:linear_MDP}
    \PP_h(\cdot \mid s,a) & = \langle \bphi(s,a), \bmu_h(\cdot) \rangle, \notag 
    \\  r_h(s,a) & = \langle \bphi (s,a), \bgamma_h \rangle,
\end{align}where $\max\Big\{\big\|\bmu_h(\cS)\big\|_2, \|\bgamma_h\|_2\Big\}\leq \sqrt{d}$ for all $h \in [H]$.

We assume that at any step $h\in[H]$, for any state-action pair $(s_h,a_h)\in\cS\times\cA$, the reward received by the agent is given by $ r_h(s_h,a_h)$.
Without loss of generality, we assume $0 \leq r_h(s,a)\leq 1$ and $\| \bphi(s,a)\|_2\leq 1$ for all $(s,a)\in \cS\times \cA$. 
We assume $\mathcal{A}$ is large but finite, while $\mathcal{S}$ is possibly infinite. 
\end{assumption}

\subsection{Cooperative Multi-agent RL with Asynchronous Communication}

% We have a group of $M$ agents. For each agent $m \in [M]$, there is an associated MDP denoted by $MDP(\cS, \cA, H, \PP_m, \br_m)$. Here $\cS$, $\cA$ are the state and action spaces, $\PP_m$ is the transition probability given by $\PP_M = \{\PP_{m,h}\}_{h \in [H]}$ where $\PP_{m,h} (s' | s,a)$ is the probability of transitioning to state $s'$ from state-action pair $(s,a)$ at step $h$. The reward function is given by $\br_m = \{r_{m,h}\}_{h \in [H]}$ where $r_{m,h}: \cS \times \cA \to [0,1]$. 
% We assume $\cA$ is finite, $\cS$ is infinite.

We assume there is a group of $M$ agents. The process proceeds in an episodic fashion, where the total number of episodes is $K$. 
At each episode $k$, there is only an active agent participating and we denote this agent by $m_k$. 
This agent adopts a policy $\pi_{m_k,k}$ and starts from an initial state $s_{k,1}$. 
For each step $h\in[H]$, agent $m_k$ picks an action $a_{k,h}\in \cA$ according to $a_{k,h} \sim \pi_{k,h}(\cdot \ | s_{k,h})$, receives a reward $r_{k,h} \sim r_{h}(s_{k,h}, a_{k,h})$, and transitions to the next state $s_{k,h+1}$. 
The episode $k$ terminates when agent $m_k$ reaches $s_{k, H+1}$ and there is zero reward at step $H+1$. Note that we consider the homogeneous agent setting, where reach agent has the same transition kernel and reward functions. 

% If there is only one agent (i.e. $M=1$), then our setting reduces to that of the single-agent linear MDP \citep{jin2020provably}. 
\paragraph{Asynchronous Communication.}
In the multi-agent setting, the agents need to communicate (i.e. share data) to collaboratively learn the underlying optimal policy while minimizing the regret. 
Without communication, the problem would reduce to $M$ separate single-agent linear MDP problems. This would lead to a worst-case regret of order $\Tilde{\mathcal{O}}(M\sqrt{K/M})$, which suffers from an extra $\sqrt{M}$ factor as compared to the $\Tilde{\mathcal{O}}(\sqrt{K})$ regret in the single-agent $K$-episode setting \citep{jin2020provably}. 
In the following sections we will show that this extra factor can be avoided at the cost of a small number of communication rounds. 

\begin{figure}[t]
    \centering
    \includegraphics[width=.7\linewidth]{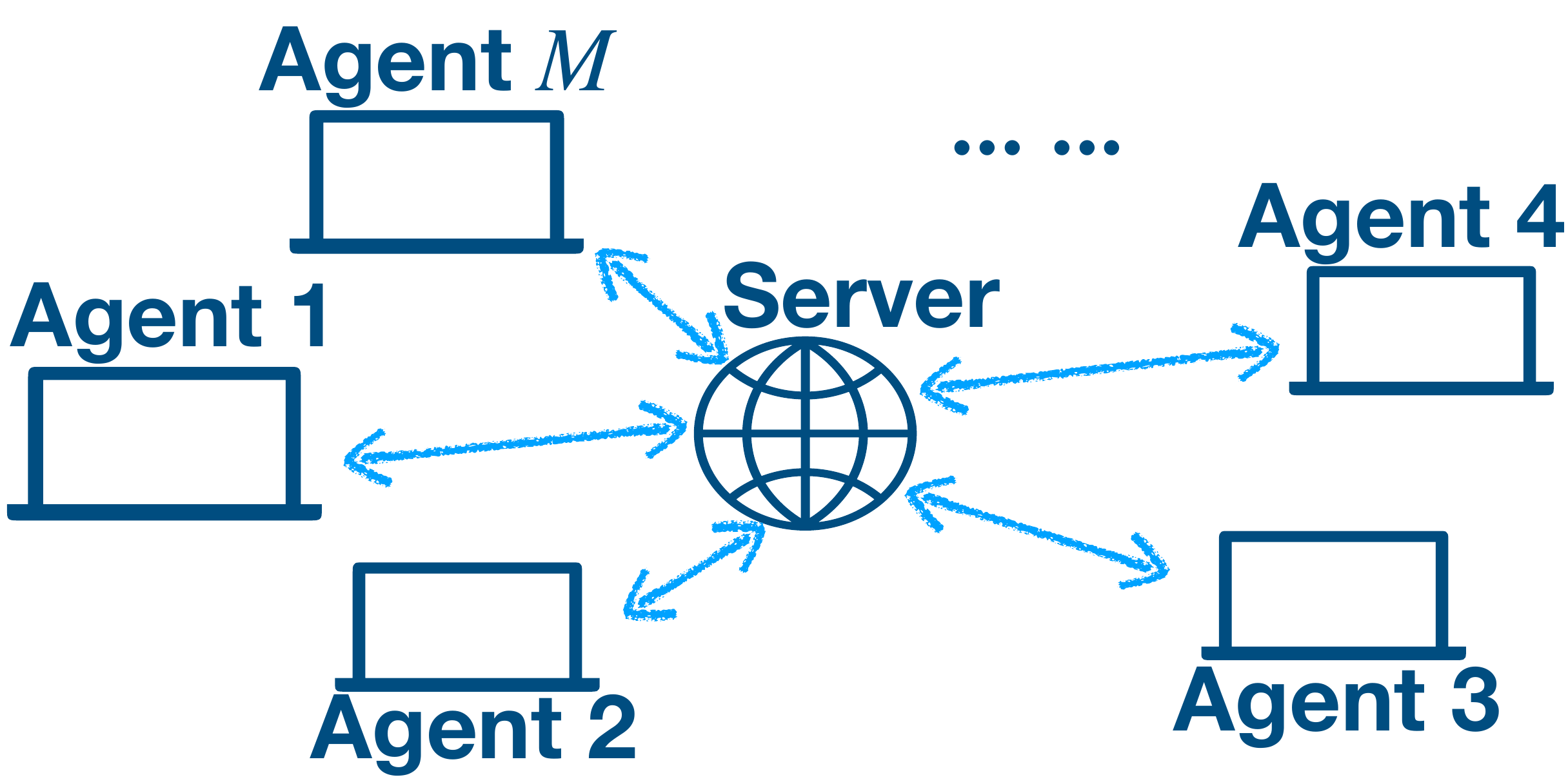}
    \caption{Illustration of Star-shaped Communication Network}
    \label{fig: communication network}
\end{figure}

We now describe our communication protocol as follows: 
In this paper we assume the existence of a central server through which all the agents can share their local data~(Figure~\ref{fig: communication network}). 
Each agent can communicate with the server by uploading local data and downloading global data from the server. 
This is also known as the star-shaped communication network \citep{wang2019distributed, dubey2020differentially, he2022simple}. 

Moreover, each agent can decide whether to trigger a communication with the server or not. 
Specifically, at the end of episode $k$, the active agent $m_k$ can choose whether to upload its local trajectory to the central server and download all the data uploaded to the server by that time. 
The communication complexity is measured by the total number of communication rounds between the agents and the server. 

Importantly, we consider an asynchronous setting satisfying the following two properties:
\begin{itemize}
    \item[(i)] Full participation or a round-robin-type participation is \emph{not} required.
    
    \item[(ii)] The communication between one agent and the server will not cause mandatory download for other agents.
\end{itemize}
This setting is much more flexible than the synchronous setting where no offline agent is allowed. 
As a comparison, in \citet{dubey2021provably}, all agents are required to participate in a round-robin fashion. 
Our setting is more general than the synchronous setting since it gives the agents the extra flexibility to decide whether to participate or not.
The pseudo-code of the communication protocol is summarized by Algorithm~\ref{alg: communication protocol}. 

\begin{algorithm}[t!]
	\caption{\texttt{Communication Protocol}}
	\label{alg: communication protocol}
	\begin{algorithmic}[1]
    \FOR{$k=1,\ldots, K$}
    \STATE Agent $m_k$ is active
    \STATE Receives $s_{k,1}$ from the environment
    % \STATE Determine the policy $\pi_{k}$
    \FOR{$h = 1, \cdots, H$}
    \STATE Take action $a_{k,h} \leftarrow \pi_{k,h} (\cdot | s_{k,h}) $ 
    \STATE Receive $s_{k,h+1} \sim \PP_{h}(\cdot|s_{k,h}, a_{k, h})$ \STATE Receive reward $r_{k,h}$
    \ENDFOR 
    \IF{Communication Triggered} 
    \STATE Send local data SERVER.
    \STATE Download from SERVER
    \STATE Update policy using all available data
    \ENDIF
    \ENDFOR
	\end{algorithmic}
\end{algorithm}

\paragraph{Communication Complexity and Switching Cost.} 
We define the communication complexity of an algorithm to be the total number of rounds of communication (i.e. one upload and download operation) between any agent and the central server. 
Note that some papers also use the number of bits to measure the communication complexity~\citep{wang2019distributed}. 
Here we follow the notation of communication complexity in \citet{dubey2021provably}.

The policy switching cost refers to the number of times the agents change their policies \citep{kalai2005efficient, abbasi2011improved}. In the RL setting where the agents choose to adopt a greedy policy according to their estimated Q-functions, the switching cost is also equal to the number of times the estimated Q-functions are updated.

\paragraph{Learning Objective.} 
For any policy $\pi=\{\pi_h\}_{h=1}^H$, we define the corresponding value functions as
\begin{align*}
    V^\pi_h (s) \coloneqq \mathbb{E} \bigg[ \sum_{h'=h}^H r_{h'}(s_{h'}, \pi_{h'}(s_{h'})) \bigg| s_h = s \bigg] , \ \forall h\in [H].
\end{align*} Since the horizon $H$ is finite and action space $\mathcal{A}$ is also finite, there exists some optimal policy $\pi^\star$ such that 
\begin{align*}
    V_h^{\pi^\star}(s) = \sup_\pi V_h^{\pi}(s),
\end{align*} and we denote $V_h^{\star} = V_h^{\pi^\star} $. Due to space limit, more definition details of the value functions and Q-functions can be found in Appendix~\ref{sec: clarification of notation}.

The objective of all agents is to collaboratively minimize the aggregated cumulative regret defined as
\begin{align}\label{eq: definition regret}
    R(K) & \coloneqq \sum_{k=1}^K \left[ V^{\star}_1 (s_{k,1}) - V_1^{\pi_{m_k,k}} (s_{k,1}) \right] , 
\end{align}where $m_k$ is the active agent in episode $k$, $\pi_{m_k, k}=\{\pi_{m_k,k,h}\}_{h=1}^H$ is the policy adopted by agent $m_k$ in episode $k$,  and $s_{k,1}$ is the initial state of episode $k$.

\section{The Proposed Algorithm}\label{sec:algorithm}

Now we proceed to present our proposed algorithm, as displayed in Algorithm~\ref{alg: main}.
After explaining the detailed design of Algorithm~\ref{alg: main}, we will also discuss possible extensions to other MDP settings in Section~\ref{sec:algorithm_extension}.
Here for notational convenience, we abbreviate $\bphi_{k,h} := \bphi(s_{k,h},a_{k,h})$.

\subsection{Algorithmic Design}
Algorithm~\ref{alg: main} adopts an execute-then-update framework on a high level: 
In each episode $k\in[K]$, there is one active agent denoted by $m = m_k$ (Line~\ref{algline: active agent} in Algorithm~\ref{alg: main}).
Here we omit the subscript $k$ for better readability.
Each episode involves an interaction phase~(Line~\ref{algline: for h begin}-\ref{algline: for h end}) and an event-triggered communication and policy update phase~(Line~\ref{algline: determinant condition}-\ref{algline: trigger else}).

\paragraph{Phase I: Interaction.}
Agent $m$ will first interact with the environment by executing the greedy policy with respect to its current Q-function estimates $\{Q_{m,k,h}\}_{h=1}^H$ (Line~\ref{algline: greedy policy}~\&~\ref{algline: take action}), and collect the data from the trajectory of the current episode~(Line~\ref{algline:next_state}~\&~\ref{algline:reward}). 
Then the agent will update its local dataset and covariance matrices (Line~\ref{algline: local dataset}~\&~\ref{algline: local covariance matrix}).

\paragraph{Phase II: Communication and Policy Update.}
The second phase involving communication and policy update is triggered by a determinant-based criterion~(Line~\ref{algline: determinant condition}). 
Once the criterion is satisfied, the agent will upload all the accumulated local data to the central server (Line~\ref{algline: upload}), and then download all the available data from the server (Line~\ref{algline: download}). 
Using this latest dataset, the agent then updates its Q-function estimates (Line~\ref{algline: update w}-\ref{algline: update Q}).

More specifically, the Q-function estimate is obtained by using backward least square value iteration following \citet{jin2020provably}: Given the estimate $Q_{m,k+1,h+1}$ for step $h+1$, we solve for $\wb_{m,k+1,h}$ that minimizes the Bellman error in terms of a ridge linear regression:
\begin{align}
    \wb_{m,k+1,h} &= \bLambda_{m,k+1,h}^{-1} \sum_{\tau \in {\cD_{m,k,h}}} \bphi(s_{\tau,h}, a_{\tau,h})\notag\\
    &\qquad \cdot \left[ r_{\tau,h} + \max_{a} Q_{m,k+1,h+1} (s_{\tau,h+1}, a) \right].\label{eq:w} 
\end{align}
In the above summation we use $\tau\in\cD_{m,k,h}$ to denote the collection of indices of all the episodes whose data is available to agent $m$ by the end of episode $k$. 
Recall from line~\ref{algline: local dataset} of Algorithm~\ref{alg: main} that the original definition of $\mathcal{D}_{m,k,h}$ is the dataset of trajectories available to agent $m$. 
Here we slightly abuse the definition of $\mathcal{D}_{m,k,h}$ to reflect the fact that $\wb_{m,k+1,h}$ is computed using only available trajectories to agent $m$ by the end of episode $k$. 
% since this $\tau\in\cD_{m,k,h}$ notation is only required in a summation of this kind in the proof and won't cause further confusion.
% \highlight{Something's wrong with $\tau\in\cD_{m,k,h}$}
The Q-function estimate for step $h$ is given by
\begin{align}
   &Q_{m,k+1,h}(\cdot,\cdot) \label{eq:Q}
    \\ &\quad=\left[ \bphi(\cdot,\cdot)^{\top} \wb_{m,k+1,h} + \Gamma_{m,k+1,h}(\cdot,\cdot)\right]_{[0, H-h+1]},\notag 
\end{align}
where $\Gamma_{m,k+1,h}$ is a bonus term that ensures the optimism of $Q_{m,k+1,h}$, and is defined by
\begin{align}
    \Gamma_{m,k+1,h} (\cdot,\cdot) = \beta \cdot \|\bphi(\cdot,\cdot)\|_{\bLambda_{m,k+1,h}^{-1}}.\label{eq:Gamma} 
    % \left[\bphi(\cdot,\cdot)^\top (\bLambda_{m,k+1,h})^{-1} \bphi(\cdot,\cdot) \right]^{1/2} \label{eq:Gamma}
\end{align}

Otherwise if the criterion is not satisfied, then the local data as well as the Q-function estimates remain unchanged for this agent (Line~\ref{algline: no update}).

% Finally, all other inactive agents' Q-function estimates and covariance matrices remain unchanged.
% to be used in possible future episodes . 

% Note that Algorithm~\ref{alg: main} 
\paragraph{Discussion on The Criterion.}
The determinant-based criterion is a common and important technique in single-agent contextual bandits~\citep{abbasi2011improved,ruan2021linear} and RL with linear function approximation \citep{zhou2021provably,wang2021provably,min2022learning}, where it is often used to reduce the policy switching cost.
While in our multi-agent linear MDP setting, we apply it to determine the appropriate time for communication and the corresponding policy update.   
The similar idea has been adopted by other works on multi-agent bandits and RL problems \citep{wang2019distributed, li2022asynchronous, he2022simple, dubey2020differentially, dubey2021provably}.

In our Algorithm~\ref{alg: main}, the criterion is adjusted by a parameter $\alpha$ that controls the communication frequency: smaller $\alpha$ indicates less frequent communication and larger $\alpha$ implies otherwise.
As will be shown in Section~\ref{sec: theory}, $\alpha$ determines the trade-off between the total number of communication rounds (or equivalently, policy updates) and the total regret of Algorithm~\ref{alg: main}. 
With a proper choice of $\alpha$, we show that our algorithm achieves a regret nearly identical to that under the single-agent setting~\citep{jin2020provably} at a low communication complexity which depends only logarithmically on $K$.

% The detail of the Q-function update (line~\ref{algline: update w}-\ref{algline: update Q}) is given by the following equations. 
% \begin{align}
%     & \wb_{m,k+1,h} \notag
%     \\ &= \bLambda_{m,k+1,h}^{-1} \sum\nolimits_{\tau=1, \tau \in {\cD_{m,k,h}}}^{k} \bphi(s_{\tau,h}, a_{\tau,h})\notag\\
%     &\qquad \cdot \left[ r_{\tau,h} + \max_{a} Q_{m,k+1,h+1} (s_{\tau,h+1}, a) \right]\label{eq:w}     
%     \\ & \notag
%     \\ 
%     &\Gamma_{m,k+1,h} (\cdot,\cdot) \notag
%     \\ & = \beta \cdot \left[\bphi(\cdot,\cdot)^\top (\bLambda_{m,k+1,h})^{-1} \bphi(\cdot,\cdot) \right]^{1/2} \label{eq:Gamma}
%     \\ & \notag
%     \\ & Q_{m,k+1,h}(\cdot,\cdot) \notag
%     \\ & =  \left[ \bphi(\cdot,\cdot)^{\top} \wb_{m,k+1,h} + \Gamma_{m,k+1,h}(\cdot,\cdot)\right]_{[0, H-h+1]} \label{eq:Q}
% \end{align}

\begin{algorithm}[t!]
	\caption{\texttt{Asynchronous Multi-agent LSVI}}
	\label{alg: main}
	\begin{algorithmic}[1]
	\STATE {\bfseries Input: }number of episodes $K$, $\beta$, $\alpha$  
    \STATE {\bfseries Initialize: }$\bLambda_{m,1,h}\leftarrow \lambda \Ib_{d\times d} $, $\wb_{m,1,h}\leftarrow \mathbf{0}$, $Q_{m,1,h} \leftarrow [ \beta [\bphi(\cdot,\cdot)^\top (\bLambda_{m,1,h})^{-1} \bphi(\cdot,\cdot) ]^{1/2}]_{[0, H-h+1]}$, $\bLambda_{m,0,h}^{\textnormal{loc}} \leftarrow \mathbf{0} $, $\cD_{m,0,h}\leftarrow \emptyset$, $\forall \ m,h \in [M]\times[H]$  
    \FOR{$k=1,\ldots, K$}
    \STATE Agent $m = m_k$ is active\alglinelabel{algline: active agent}
    \STATE Receive $s_{k,1}$ from the environment
    \FOR{$h = 1, \ldots, H$}\alglinelabel{algline: for h begin}
    \STATE $\pi_{m,k,h}(\cdot) \leftarrow \argmax_{a\in\cA} Q_{m,k,h} (\cdot, a)$ \alglinelabel{algline: greedy policy}
    \STATE Take action $a_{k,h} \sim \pi_{m,k,h} (\cdot | s_{k,h}) $ \alglinelabel{algline: take action}
    \STATE Receive next state $s_{k,h+1} \sim \PP_{h}(\cdot|s_{k,h}, a_{k, h})$\alglinelabel{algline:next_state} 
    \STATE Receive reward $r_{k,h}$\alglinelabel{algline:reward}
    \STATE $\cD_{m,k,h} \leftarrow \cD_{m,k-1,h}\cup \{ s_{k,h}, a_{k,h}, s_{k,h+1}, r_{k,h}\}$\alglinelabel{algline: local dataset}
    \STATE $\bLambda_{m,k,h}^{\textnormal{loc}} \leftarrow \bLambda_{m,k-1,h}^{\textnormal{loc}} + \bphi_{k,h} \bphi_{k,h}^{\top}$ \alglinelabel{algline: local covariance matrix}
    % \STATE $\bLambda_{m,k,h}^{\textnormal{loc}} \leftarrow \bLambda_{m,k-1,h}^{\textnormal{loc}} + \bphi(s_{k,h}, a_{k,h}) \bphi(s_{k,h}, a_{k,h})^{\top}$ 
    \ENDFOR \alglinelabel{algline: for h end}
    \IF{$\exists h$ s.t. $\frac{\det(\bLambda_{m,k,h}+\bLambda_{m,k, h}^{\text{loc}})}{\det(\bLambda_{m,k,h})} >  1+\alpha $}\alglinelabel{algline: determinant condition}
    \STATE $Q_{m, k+1, H+1} \leftarrow 0$
    \FOR{$h = H, H-1,\cdots, 1$}
    \STATE Agent $m$ sends local data to SERVER:
    \STATE $\bLambda_{k,h}^{\textnormal{ser}} \leftarrow \bLambda_{k,h}^{\textnormal{ser}} + \bLambda_{m,k,h}^{\textnormal{loc}} $, $\cD_{k,h}^{\textnormal{ser}} \leftarrow \cD_{k,h}^{\textnormal{ser}} \cup \cD_{m, k, h}$\alglinelabel{algline: upload}
    \STATE SERVER sends global data back to agent $m$:
    \STATE $\bLambda_{m,k+1,h} \leftarrow \bLambda_{k,h}^{\text{ser}}$, $\cD_{m,k,h} \leftarrow \cD_{k,h}^{\textnormal{ser}} $\alglinelabel{algline: download}
    \STATE Update estimate $\wb_{m,k+1,h}$ by \eqref{eq:w} \alglinelabel{algline: update w}
    \STATE Update bonus $\Gamma_{m,k+1,h}$ by \eqref{eq:Gamma}
    \STATE Update Q-function estimate $Q_{m,k+1,h}$ by \eqref{eq:Q} \alglinelabel{algline: update Q}
    \ENDFOR
    \STATE Reset $\bLambda_{m,k,h}^{\textnormal{loc}} \leftarrow \textbf{0}$, $\forall \ h \in [H]$
    \ELSE \alglinelabel{algline: trigger else}
    \STATE $Q_{m,k+1,h}\leftarrow Q_{m,k,h}$, $\bLambda_{m,k+1,h} \leftarrow \bLambda_{m,k,h}$, $\bLambda_{k+1,h}^{\textnormal{ser}}\leftarrow \bLambda_{k,h}^{\textnormal{ser}}$, $\cD_{k+1,h}^{\textnormal{ser}} \leftarrow \cD_{k,h}^{\textnormal{ser}}$, $\forall \ h \in [H]$\alglinelabel{algline: no update}
    \ENDIF
    \FOR{all other inactive agents $m'\neq m$}
    \STATE $Q_{m',k+1,h}\leftarrow Q_{m',k,h}$, $\bLambda_{m',k+1,h} \leftarrow \bLambda_{m',k,h}$, $\forall \ h \in [H]$
    \ENDFOR
    \ENDFOR
	\end{algorithmic}
\end{algorithm}

\subsection{Extension to Other MDP Settings}\label{sec:algorithm_extension}

We remark that though designed for linear MDPs, Algorithm~\ref{alg: main} can be easily extended to other MDP settings.

Note that for any tabular MDP with small state and action space, we can represent it using the one-hot feature mapping as discussed in Example 2.1 in \citet{jin2020provably}.
Thus Algorithm~\ref{alg: main} can be applied directly to tabular MDPs, and the determinant-based criterion in Line~\ref{algline: determinant condition} would become a criterion based on the visitation count for every state-action pair.
However, we anticipate that Theorem~\ref{thm: upper bound} would produce in this case an regret upper bound that is suboptimal in $|\cS|$ and $|\cA|$, which is possibly due to that the one-hot feature mapping is not a good representation.

Moreover, the algorithm design can also be applied to linear mixture MDPs where we would have a tenary feature mapping $\phi: \cS\times\cA\times\cS\to\RR^d$.
Then for example, the \texttt{UCRL-VTR} algorithm~\citep{jia2020model,ayoub2020model} can be adapted to the asynchronous cooperative setting by designing a similar communication criterion based on the covariance matrix defined over $\phi_V(\cdot,\cdot):= \sum_{s\in\cS} \phi(s,\cdot,\cdot) V(s)$, instead of the $\bLambda$ defined over only the feature mapping in our case of linear MDPs.

There have also been other more advanced algorithms for linear MDPs~\citep{hu2022nearly,he2022nearly,agarwal2022vo} that exploit variance information to further reduce the dependence of the regret bound on problem parameters.
We leave it as future work to develop and analyze variance-aware variants of Algorithm~\ref{alg: main}.

% \highlight{[Question: our algorithm design can also be applied to linear 
% mixture MDP, as well as tabular MDP?]}

% \highlight{Answer: Technically yes. But would suffer from suboptimal dependence on $d$.}

\section{Main Results}\label{sec: theory}

We now present the main theoretical results.
We provide the regret upper bound for Algorithm~\ref{alg: main} in Theorem~\ref{thm: upper bound}, and compare it with existing related results.
Then as a complement, in Theorem~\ref{thm:lower_bound}, we provide a lower bound on the communication complexity for cooperative linear MDPs.

\subsection{Regret Upper Bound}

The following theorem provides the regret upper bound of Algorithm~\ref{alg: main}. 

\begin{theorem}[Regret Upper Bound]\label{thm: upper bound}
    Under Assumption~\ref{assump: linear_MDP}, there exists some universal constant $c_\beta$ such that by choosing 
    \begin{align*}
        \beta = c_\beta dH \tilde{C} \left[  \log\left( \frac{2+K}{\delta\min\{1,\lambda,\alpha\lambda\}} \right) + \log\left( Hd \tilde{C} \right) \right], 
    \end{align*}where $\tilde{C} \coloneqq M\sqrt{\alpha}+\sqrt{1+M\alpha}$, then with probability at least $1-\delta$, 
    the regret of Algorithm~\ref{alg: main} can be bounded as 
    \begin{align*}
        \mathcal{O}\left(\beta\sqrt{1+M\alpha} H\sqrt{dK \log\left(2dK/\min\{1,\lambda\}\delta\right)}\right) . 
    \end{align*}
    Moreover, the communication complexity and policy switching cost (in number of rounds) are upper bounded by 
    \begin{align*}
        \mathcal{O}\left(dH (M+1/\alpha)\log(1 + K/\lambda d)\right).
    \end{align*}
\end{theorem}

\begin{remark}
    Theorem~\ref{thm: upper bound} indicates that by setting the parameters $\alpha = 1/M^2$ and $\lambda = 1$ in Algorithm~\ref{alg: main}, the regret upper bound can be simplified to $\tilde{\mathcal{O}}(d^{3/2}H^2\sqrt{K})$, and the communication complexity is bounded by $\tilde{\mathcal{O}}(dHM^2)$.
\end{remark}

\begin{remark}
    % \todos{compare the result with that in \cite{dubey2021provably}}
    We compare our upper bound with the best known result by \citet{dubey2021provably}. In their Theorem~1, \citet{dubey2021provably} present an $\tilde{\mathcal{O}}(d^{3/2}H^2\sqrt{M T})$ regret upper bound for the homogeneous-agent setting (i.e. the same transition kernel and reward functions shared among all agents), which is identical to the multi-agent linear MDP setting considered in our paper. However,  their communication protocol is synchronous in a round-robin fashion (see line 4 of their Algorithm 1), and thus $MT$ under their setting is equal to our $K$. Therefore, by Theorem~\ref{thm: upper bound}, our regret upper bound for the asynchronous setting matches that for the synchronous setting. 
\end{remark}

\begin{remark}
    Our result generalizes that of the multi-agent linear bandit setting \citep{he2022simple} and the single-agent linear MDP setting \citep{jin2020provably}. Specifically, with $H=1$, our upper bound becomes $\tilde{\mathcal{O}}(d\sqrt{K})$ and the number of communication rounds becomes $\tilde{\mathcal{O}}(d M^2)$. 
    Note that we save a $\sqrt{d}$ factor in the regret bound compared to the original $d^{3/2}$ dependence from Theorem~\ref{thm: upper bound} since there is no covering issue when $H=1$.
    Both the regret and the communication reduce to those under the bandit setting \citep{he2022simple}. When $M = 1$, our regret reduces to $\Tilde{\mathcal{O}}(d^{3/2}H^2 \sqrt{K})$, matching that of \citet{jin2020provably}. 
\end{remark}

\subsection{Regret Lower Bound}

\begin{theorem}[Regret Lower Bound]\label{thm:lower_bound}
    Suppose $d,H\ge 2$ and number of episodes $K\ge dM$, then
for any algorithm \textbf{Alg} with expected communication complexity less than 
$ dM/11400$, there exist a linear MDP, such that the expected regret for algorithm \textbf{Alg} is at least 
$\Omega(H\sqrt{dMK})$.
\end{theorem}
\begin{remark}
     Theorem \ref{thm:lower_bound} suggests that, for any algorithm \textbf{Alg} with communication complexity $o(dM)$, the regret is no better than $\Omega(\sqrt{MK})$. On the other hand, if each agent perform the LSVI-UCB algorithm \citep{jin2020provably}, the total regret of $M$ agents is upper bounded by $\sum_{m=1}^M  
\tilde  O(\sqrt{K_m}) =  \tilde  O(\sqrt{MK})$, where $K_m$ is the number of episodes that agent $m$ is active. Thus, in order to improve the performance through collaboration and remove the dependency on the number of agent $M$, an $\Omega(dM)$ 
communication complexity is necessary.
\end{remark}
\begin{remark}
    Though Theorem \ref{thm:lower_bound} requires the number of stage $H\ge 2$, it is not difficult to extend the result for $H=1$ with stochastic reward function $r_h(s,a)$. In this situation, Theorem \ref{thm:lower_bound} will reduce to bandit problem with adversarial contexts and improves the communication complexity in \citet{he2022simple} with a factor of $d$. We also compare our result with the communication lower bound in \citet{amani2022distributed}. In this work, they measured the communication complexity by bits, which is strictly larger than our definition, and also provided a $\Omega(dM)$ communication complexity (in bits) for stochastic contexts.
\end{remark}
%\highlight{[Compared with other lower bounds in the bandit setting]}

\section{Overview of the Analysis}
In this section, we discuss the technical challenges of analyzing Algorithm~\ref{alg: main} and our solutions.

\subsection{Technical Challenges}
The asynchronous communication protocol causes a unique challenge in the theoretical analysis. 
To illustrate the challenge, let us first recall the synchronous setting with a round-robin-type participation, as studied in \citet{dubey2021provably}. 
Note that under this setting, the order of participation is fixed.
This implies that if $\bphi_{k,h}$ is uploaded to the server, then for all $k' < k$, the vectors $\bphi_{k',h}$ must also have already been uploaded to the server. 
As a sharp comparison, the above important condition is no longer satisfied under the asynchronous setting.
We name the violation of this condition the \textbf{information asymmetry} issue. 

Technically, this issue causes two consequences. 
Recall from the analysis of \texttt{LSVI-UCB} that the final regret bound depends on two technical lemmas: the concentration of self-normalized martingales, and the elliptical potential lemma~\citep{abbasi2011improved,jin2020provably}. 
The first lemma determines the width of the confidence region (i.e. $\beta$), and the second is crucial for bounding the sum of the bonus terms (i.e. $\sum_{k=1}^K \|\bphi_{k,h}\|_{\bLambda_{m,k,h}^{-1}}$).  
Both lemmas require a well-defined and fixed order of $\bphi_{k,h}$ vectors in the collected data in order to be applied.
Unfortunately, such an order does not exist in the asynchronous setting, since the data receiving process by the server and every agent is stochastic. 
The analysis of this stochastic process is also prohibitive because our agents have full freedom to decide whether to participate.
Therefore, this arbitrary pattern in the data collected by Algorithm~\ref{alg: main} forbids the directly application of these two tools. 

To address this information asymmetry issue, we develop a novel form of the self-normalized martingale concentration lemma (Lemma~\ref{lemma: self norm fixed V}), and an asynchronous elliptical potential lemma (Lemma~\ref{lemma: async elliptical potential}).  
The main idea is a refined analysis of the local covariance matrix $\bLambda_{m,k,h}$ and the universal covariance matrix and their comparison under the partial ordering defined by matrix positive definiteness. 
In the remaining of this section, we overview some key steps and the definition of some important quantities behind the upper bound in Theorem~\ref{thm: upper bound}. 
The full details are included in Appendix~\ref{sec: proof of upper bound}.

\subsection{Key Ingredients of the Proof} 

For any agent $m\in[M]$ and episode $k\in[K]$, define the following indices:
\begin{itemize}
    \item $m_k$: the active agent of episode $k$. Note that in Algorithm~\ref{alg: main} we use $m$ instead of $m_k$ due to space limit.  
    \item $t_k(m)$: for any agent $m$, $t_k(m)\leq k$ is the last episode when agent $m$ adopts a newly updated a policy by the end of episode $k$. If no policy updating has been conducted by the end of episode $k$, then by default $t_k(m)=1$.
    
    % \item $n_k(m)$: $n_k(m)<t_k(m)$ is the most recent episode before $k$ when the agent $m$ updates its policy. This newly updated policy is executed for the first time at episode $t_k(m)$. If no policy updating has been conducted before episode $k$ by agent $m$, then by default $n_k(m)=0$.

    % \item $N_k(m)$: $N_k(m)\leq k$ is the last episode that agent $m$ participates. If agent $m$ participates at episode $k$, then $N_k(m) = k$. 
\end{itemize}
For the participating agent $m_k$ in episode $k$, its adopted Q-function $Q_{m_k,k,h}$ would be equal to $Q_{m_k,t_k,h}$, for all $h\in[H]$, according to the definition. 
In the following, we may write $t_k = t_k(m_k)$ whenever there is no confusion. 
A comprehensive clarification of notation is also provided in Appendix~\ref{sec: clarification of notation}.

\paragraph{Regret Decomposition.}
By Definition \ref{eq: definition regret}, the regret is
\begin{align}\label{eq: regret decomposition 1}
    R(K) & \coloneqq \sum_{k=1}^K \left[ V^{\star}_1 (s_{k,1}) - V_1^{\pi_{m_k,k}} (s_{k,1}) \right] \notag
    \\ & \leq \sum_{k=1}^K \left[ V_{m_k,k,1}(s_{k,1}) - V_1^{\pi_{m_k,k}} (s_{k,1}) \right] 
    \\ & = \sum_{k=1}^K \left[ V_{m_k,t_k(m_k),1}(s_{k,1}) - V_1^{\pi_{m_k,k}} (s_{k,1})  \right]. \notag
\end{align}
The inequality is from the following optimism property, 
which is standard for UCB-type algorithms.

\begin{lemma}[Optimism]\label{lemma: optimism}
    Under the setting of Theorem~\ref{thm: upper bound}, on the event of Lemma~\ref{lemma: self normal term bound}, for all $k\in[K]$, $h \in [H]$, and $(s,a) \in \mathcal{S}\times\mathcal{A}$, we have $Q^\star_h(s,a) \leq Q_{m_k,k,h}(s,a)$.
\end{lemma}

\begin{proof}[Proof of Lemma~\ref{lemma: optimism}]
    See Appendix~\ref{sec: proof of optimism}.
\end{proof}

The last step in \eqref{eq: regret decomposition 1} follows from the definition of $t_k(m_k)$.
We then further decompose the terms as  
\begin{align*}
    & V_{m_k,t_k,h}(s_{k,h}) - V_h^{\pi_{m_k,k}} (s_{k,h})
    \\ & = Q_{m_k, t_k, h} (s_{k,h}, a_{k,h}) - Q_h^{\pi_{m_k,k}} (s_{k,h}, a_{k,h}) 
    \\ & \leq \bphi_{k,h}^\top \wb_{m_k, t_k,h} + \beta\sqrt{\bphi_{k,h}^\top \bLambda^{-1}_{m_k, t_k,h} \bphi_{k,h}} - \bphi_{k,h}^\top \wb^{\pi_{m_k,k}}_h.
\end{align*}

To analyze the above, we establish the following result.
\begin{lemma}\label{lemma: parameter distance}
    Suppose we choose $\beta$ as
    \begin{align*}
        &\beta = c_\beta Hd \tilde{C} \left[  \log\left( \frac{2+K}{\delta\min\{1,\lambda,\alpha\lambda\}} \right) + \log\left( Hd \tilde{C} \right) \right],
    \end{align*} where $\tilde{C} \coloneqq M\sqrt{\alpha}+\sqrt{1+M\alpha}$ and $c_\beta$ is some universal constant. 
    For any fixed policy $\pi$, on the event of Lemma~\ref{lemma: self normal term bound}, for any $k\in[K]$, $h\in[H]$, and $(s,a)\in \mathcal{S}\times\mathcal{A}$, it holds that  
    \begin{align*}
        &\left| \bphi(s,a)^\top \left( \wb_{m_k,t_k,h} - \wb^\pi_h\right) - \PP_h \left[V_{m_k,t_k,h+1} - V^\pi_{h+1}\right] (s,a) \right| 
        \\ & \leq \beta \sqrt{\bphi(s, a)^\top \bLambda^{-1}_{m_k, t_k,h} \bphi(s,a)}.
    \end{align*}
\end{lemma}
% \vspace{-0.09in}
\begin{proof}[Proof of Lemma~\ref{lemma: parameter distance}]
    See Appendix~\ref{sec: proof of parameter distance}. 
\end{proof}

\vspace{-0.08in}
\textbf{Taming Information Asymmetry.} Lemma~\ref{lemma: parameter distance} serves a purpose similar to Lemma B.4 in \citet{jin2020provably}. 
However, its proof is more involved due to the discrepancy between $\bLambda_{m_k,t_k,h}$ and $\lambda \Ib + \sum_{k'=1}^{k-1} \bphi_{k',h}$ under the asynchronous setting. 
In other words, a random proportion of information is missing from the covariance matrix $\bLambda_{m_k,t_k,h}$, causing the information asymmetry issue. 

To circumvent this issue, we establish a delicate comparison of the covariance matrices (\ref{lemma: matrix comparison}) via several auxiliary matrices (Appendix~\ref{sec: definition of auxiliary matrices}). 
By doing so, we can bound the discrepancy between each $\bLambda_{m_k,t_k,h}$ and the full information matrix $\lambda \Ib + \sum_{k'=1}^{k-1} \bphi_{k',h}$, and then apply the classical concentration argument for self-normalized martingales.

Lemma~\ref{lemma: parameter distance} further allows us to apply the standard recursive relation for LSVI-type algorithms \citep{jin2020provably}. 

\begin{lemma}[Recursion]\label{lemma: recursion}
    Define $\xi_{k,h} = V_{m_k, t_k, h}(s_{k,h}) - V_h^{\pi_{m_k,k}}(s_{k,h})$.
    On the event of Lemma~\ref{lemma: self normal term bound}, it holds that 
    \begin{align*}
        \xi_{k,h}  & \leq \xi_{k, h+1} + \left ( \mathbb{E}\left[ \xi_{k, h+1} \middle| s_{k,h}, a_{k,h} \right] - \xi_{k,h+1} \right) 
        \\ & \qquad + 2\beta \sqrt{\bphi_{k,h} \bLambda_{m_k,t_k,h}^{-1}\bphi_{k,h}} . 
    \end{align*}
\end{lemma}

\begin{proof}[Proof of Lemma~\ref{lemma: recursion}]
    See Appendix~\ref{sec: proof of recursion}.
\end{proof}

\vspace{-0.09in}
\textbf{Asynchronous Elliptical Potential Lemma.}
Finally, Lemma~\ref{lemma: recursion} allows us to bound the regret by the sum of bonus terms. 
However, the standard elliptical potential lemma \citep{abbasi2011improved} still does not apply due to the information asymmetry issue. 
To this end, we proposed an asynchronous elliptical potential lemma to facilitate the analysis.

\begin{lemma}[Asynchronous Elliptical Potential]\label{lemma: async elliptical potential}
    Let 
    \begin{align*}
        B_h = \sum_{h=1}^H 2 \beta \sqrt{\bphi(s_{k,h}, a_{k,h}) \bLambda_{m_k,t_k,h}^{-1}\bphi(s_{k,h}, a_{k,h})}.
    \end{align*}
    Under the same assumption of Theorem~\ref{thm: upper bound}, it holds that 
    \begin{align*}
        & \sum_{k=1}^K \min\left\{V_{m_k,t_k(m_k),1}(s_{k,1}) - V_1^{\pi_{m_k,k}} (s_{k,1}), B_h \right\}
        \\ & \leq \mathcal{O}(\beta\sqrt{1+M\alpha} H\sqrt{dK \log\left(2dK/\min\{1,\lambda\}\delta\right)}).
    \end{align*} 
\end{lemma}

\begin{proof}[Proof of Lemma~\ref{lemma: async elliptical potential}]
    See Appendix~\ref{sec: proof of async elliptical potential}.
\end{proof} 

With all the above steps, we can finally establish the regret upper bound in Theorem~\ref{thm: upper bound}. 
The remaining details are provided in Appendix~\ref{sec: proof of upper bound}.

\section{Conclusion and Future Work}
% In this paper we study cooperative multi-agent RL with asynchronous communication under linear function approximation.
% We propose a provably efficient algorithm for cooperative linear MDPs, and provide a novel theoretical analysis resolving the challenge of information asymmetry induced by asynchronous communication.
% We also provide a lower bound on the communication complexity for cooperative multi-agent RL.

In this paper
we propose a provably efficient algorithm for cooperative multi-agent RL with asynchronous communication, and provide a novel theoretical analysis resolving the challenge of information asymmetry induced by asynchronous communication.
We also provide a lower bound on the communication complexity for such a setting.

There are several possible directions for future work.
Moreover, the current lower bound in Theorem~\ref{thm:lower_bound} is arguably not tight, and it requires novel construction to exhibit the fundamental trade-off between reducing communication complexity and lowering regret.
Also, for the asynchronous communication model, it is important to incorporate other aspects such as agent heterogeneity, environment non-stationarity, privacy and more general function approximation.
% First, it is possible to incorporate heterogeneity.
% Finally, it is also an important question to guarantee privacy.
% \highlight{[Future work: P2P, heterogeneity, lower bound, general function approximation; privacy, federated]}

\section*{Acknowledgments and Disclosure of Funding}
We thank the anonymous reviewers and area chair for their helpful comments. JH and QG are partially supported by the National Science Foundation CAREER Award 1906169 and the Sloan Research Fellowship.  The views and conclusions contained in this paper are those of the authors and should not be interpreted as representing any funding agencies.

\bibliography{reference}
\bibliographystyle{icml2023}

%%%%%%%%%%%%%%%%%%%%%%%%%%%%%%%%%%%%%%%%%%%%%%%%%%%%%%%%%%%%%%%%%%%%%%%%%%%%%%%
%%%%%%%%%%%%%%%%%%%%%%%%%%%%%%%%%%%%%%%%%%%%%%%%%%%%%%%%%%%%%%%%%%%%%%%%%%%%%%%
% APPENDIX
%%%%%%%%%%%%%%%%%%%%%%%%%%%%%%%%%%%%%%%%%%%%%%%%%%%%%%%%%%%%%%%%%%%%%%%%%%%%%%%
%%%%%%%%%%%%%%%%%%%%%%%%%%%%%%%%%%%%%%%%%%%%%%%%%%%%%%%%%%%%%%%%%%%%%%%%%%%%%%%
\newpage
\appendix
\onecolumn

\section{Clarification of Notation}\label{sec: clarification of notation}

In this section, we give a comprehensive clarification on the notation used in the algorithm and the analysis. 

Throughout the paper, we use $\mathcal{O}(\cdot)$ to hide problem-independent universal constants and $\tilde{\mathcal{O}}(\cdot)$ to further hide logarithmic factors. 
We use $(\cdot)_{[a,b]}$ to denote the truncation of values into the range $[a,b]$.

% We use $\pi = \{\pi_h\}_{h=1}^H$ to denote a policy, where each $\pi_h$ is defined to be a mapping from $\cS$ to a distribution $\Delta_\cA$ on $\cA$. Therefore, for any $h\in[H]$ and $s_h \in \cS$, $\pi_h(\cdot| s_h)$ denotes a probability distribution on $\cA$. Note that because $\cA = \Upsilon \times \cX \times \cT$, the policy $\pi$ is a joint policy. We may slightly abuse the notation in the paper and refer to $\pi$ as the policy restricted on $\Upsilon$ only, whenever it is clear from the context. In such case, we refer to $\pi_h(\cdot| s_h)$ as a distribution on $\Upsilon$ only. 

We also present the following table of notations. The $\pi$ in the superscript can be replaced by $\pi_k$ or $\pi_\star$, where the former refers to the policy in episode $k$, and the latter refers to the optimal policy.
% If the joint policy $\sigma$ is equivalent to a product of two separate policies for each player, i.e., $\sigma(a,b|s) = \pi(a|s)\times \nu(b|s)$, then we can replace $\sigma$ by $\pi, \nu$.

\begin{table}[!ht]
\caption{Notation}
\centering
\renewcommand*{\arraystretch}{1.5}
\begin{tabular}{ >{\centering\arraybackslash}m{3.5cm} | >{\centering\arraybackslash}m{11.1cm} } 
\hline\hline
Notation & Meaning \\ 
\hline

$m_k$ & the active agent in episode $k$\\

{$\pi_{m,k}=\{\pi_{m,k,h}\}_{h=1}^H$} & the policy of agent $m$ at episode $k$ (regardless of agent $m$ being active or not)\\ 

\hline

{$\{Q_{m,k,h}\}_{h=1}^H$} & Q-functions of agent $m$ at episode $k$ in Algorithm~\ref{alg: main}
\\

{$\{V_{m,k,h}\}_{h=1}^H$} & Value functions of agent $m$ at episode $k$ in Algorithm~\ref{alg: main}, where $V_{m,k,h}(\cdot)=\argmax_a Q_{m,k,h}(\cdot,a)$
\\ 

\hline
$\{V_h^\pi\}_{h=1}^H$ & Value functions under policy $\pi$\\

\hline
$\bphi_{k,h}$ & $\bphi_{k,h} = \bphi(s_{k,h}, a_{k,h})$ for $k\in[K]$ and $h\in[H]$\\

\hline
$\wb_{m,k,h}, \ \bLambda_{m,k,h}$ & underlying parameter and covariance matrix of $Q_{m,k,h}$ in Algorithm~\ref{alg: main}\\ 

\hline \hline
\end{tabular}
\label{tab:notation}
\end{table}

\paragraph{Indices of available episodes} 
Recall from line~\ref{algline: local dataset} and \ref{algline: download} of Algorithm~\ref{alg: main} that the original definition of $\mathcal{D}_{m,k,h}$ is the dataset of trajectories available to agent $m$ by the end of episode $k$. 
However, in our analysis we may use $\tau\in\cD_{m,k,h}$ to denote the collection of indices of all the episodes whose data is available to agent $m$ at the beginning of episode $k$. 
That is, $\tau \in \mathcal{D}_{m,k,h}$ refers to all the episodes whose trajectories are available to agent $m$ by the end of episode $k$.
For example, in \eqref{eq:w}, we use the summation over the indices $\tau \in \mathcal{D}_{m,k,h}$ to reflect that the parameter $\wb_{m,k+1,h}$ is computed using only available trajectories (either the agent's own local trajectories or downloaded ones) by the end of episode $k$. 
This is a slight abuse of the definition of $\mathcal{D}_{m,k,h}$, 
since this $\tau\in\cD_{m,k,h}$ notation is only required in a summation of this kind in the proof and won't cause further confusion. 

\paragraph{Value and Q-functions} For any policy $\pi=\{\pi_h\}_{h=1}^H$, we define the corresponding value functions as
\begin{align*}
    V^\pi_h (s) \coloneqq \mathbb{E} \left[ \sum_{h'=h}^H r_{h'}(s_{h'}, \pi_{h'}(s_{h'})) \middle| s_h = s \right] , \ \forall \ h\in [H].
\end{align*} The Q-functions are defined as
\begin{align*}
    Q_h^\pi(s,a) = r_h(s,a) + \mathbb{E} \left[ \sum_{h'=h+1}^H r_{h'}(s_{h'}, \pi_{h'}(s_{h'})) \middle| s_h = s , a_h = a\right] , \ \forall \ h\in [H].
\end{align*}Since the horizon $H$ is finite and action space $\mathcal{A}$ is also finite, there exists some optimal policy $\pi^\star$ such that 
\begin{align*}
    V_h^{\pi^\star}(s) = \sup_\pi V_h^{\pi}(s). 
\end{align*}We denote $V_h^{\pi^\star} = V_h^{\star}$. Furthermore, the above definition implies the following Bellman equations
\begin{align*}
    Q_h^\pi(s,a) = r_h(s,a) + \mathbb{P}_h V_{h+1}^\pi (s,a), \ V_{h}^\pi (s) = Q_{h}^\pi (s, \pi_h(s)) , \ \forall \ h \in [H],
\end{align*}where $V_{H+1}^\pi(\cdot)=0$.

\paragraph{Multi-value quantities.} The following quantities from Algorithm~\ref{alg: main} can possibly take two different values in an episode due to the policy update. In our analysis, we assume they refer to the values at the end of the episode $k$, unless otherwise stated.
\begin{itemize}
    \item $\cD_{m, k,h}$; $\cD_{k,h}^{\textnormal{ser}}$; $\bLambda_{k,h}^{\textnormal{ser}}$; $\bLambda_{m,k,h}^{\textnormal{loc}}$. 
\end{itemize}

\paragraph{Indices of episodes.} The following indices of episodes are necessary to the analysis under the asynchronous setting:
\begin{itemize}
    \item $t_k(m)$: $t_k(m)\leq k$ is the last episode when agent $m$ adopts a newly updated a policy by the end of episode $k$. If no policy updating has been conducted by the end of episode $k$, then by default $t_k(m)=1$.
    
    \item $n_k(m)$: $n_k(m)<t_k(m)$ is the most recent episode before $k$ when the agent $m$ updates its policy. This newly updated policy is executed for the first time at episode $t_k(m)$. If no policy updating has been conducted before episode $k$ by agent $m$, then by default $n_k(m)=0$.

    \item $N_k(m)$: $N_k(m)\leq k$ is the last episode that agent $m$ participates up until the end of episode $k$. For example, if agent $m$ participates in episode $k$, then $N_k(m) = k$. 
\end{itemize}

The above definition implies $0 \leq n_k(m) < t_k(m) \leq N_k(m) \leq k$.

\subsection{Auxiliary Matrices}\label{sec: definition of auxiliary matrices}

We further define a few notations that will be used extensively in the proof.

\paragraph{Universal information.} We define the following matrix of universal information up to the beginning of episode $k$: 
\begin{align}\label{eq: universal information}
    \bLambda_{k,h}^{\textnormal{all}} = \lambda \Ib + \sum_{\tau=1}^{k-1} \bphi_{\tau,h} \bphi_{\tau,h}^{\top} , \ \forall \ h \in [H]. 
\end{align}

\paragraph{Personal information.} 
We define the uploaded information by agent $m$ until episode $k$ as 
\begin{align}\label{eq: def up information}
    \bLambda_{m,k,h}^{\textnormal{up}} = \sum_{\tau=1, m_\tau = m}^{n_k(m)} \bphi_{\tau,h}\bphi_{\tau,h}^{\top}, \ \forall \ h \in [H].
\end{align}

Since quantities such as $\bLambda_{m,k,h}^{\textnormal{loc}}$ can possibly take two different values during episode $k$ due to the policy update, in the following, we assume all these quantities refer to the value at the end of each episode. The matrix $\bLambda_{m,k,h}^{\textnormal{loc}}$ can be rewritten as 
\begin{align}\label{eq: loc information}
    \bLambda_{m,k,h}^{\textnormal{loc}} & = \sum_{\tau=n_k(m)+1, m_\tau = m}^{k} \bphi_{\tau,h}\bphi_{\tau,h}^{\top}, \ \forall \ h \in [H].
\end{align}

\section{Proof of Regret Upper Bound}\label{sec: proof of upper bound}

\subsection{Basic Properties of the LSVI-type Algorithm}

In this section, we list a few basic lemmas for our LSVI-type algorithm. Most of these lemmas are modified from those in \citet{jin2020provably}. These lemmas are crucial to our regret upper bound. 

\begin{lemma}[Lemma B.1 in \citealt{jin2020provably}]\label{lemma: bound norm of w pi}
    Under Assumption~\ref{assump: linear_MDP}, for any policy $\pi$, for any $h \in [H]$, let $\wb_h^\pi$ be such that $Q_h^\pi (\cdot,\cdot) = \langle \bphi(\cdot, \cdot), \wb_h^\pi \rangle $. Then for all $h \in [H]$, it holds that
    \begin{align*}
        \| \wb_h^\pi \| \leq 2H\sqrt{d}.
    \end{align*}
\end{lemma}

\begin{lemma}\label{lemma: norm w estimate}
    Under Assumption~\ref{assump: linear_MDP}, for any $k \in [K]$ and $h \in [H]$, the estimated parameter $\wb_{m, k, h}$ in Algorithm~\ref{alg: main} satisfies
    \begin{align*}
        \left\|\wb_{m, k, h} \right\| \leq 2 H \sqrt{dk/\lambda}. 
    \end{align*}
\end{lemma}

\begin{proof}[Proof of Lemma~\ref{lemma: norm w estimate}]
    See Appendix~\ref{sec: proof of norm w}.
\end{proof}

Recall the auxiliary matrices defined in Appendix~\ref{sec: definition of auxiliary matrices}. 
The following result describes the partial ordering between them. 

\begin{lemma}[Covariance matrix ordering]\label{lemma: matrix comparison}
    Under the setting of Theorem~\ref{thm: upper bound}, it holds that 
    \begin{align}\label{eq: matrix comparison up}
        \lambda \Ib + \sum_{m'\in [M]} \bLambda_{m', k , h}^{\textnormal{up}} \succeq \frac{1}{\alpha} \bLambda_{m, k , h}^{\textnormal{loc}} , \ \forall \ k,h,m \in [K]\times [H]\times[M]. 
    \end{align}Furthermore, for some $1<\underbar{t} \leq \bar{t} \leq K$, suppose agent $m$ is the only participating agent within these episodes (i.e. $m_k = m$ for all $k \in [\underbar{t} , \bar{t}]$), and agent $m$ communicates with the server only at episode $k=\underbar{t}$ during $[\underbar{t} , \bar{t}]$. Then for all $k \in [\underbar{t}+1 , \bar{t}]$, it holds that 
    \begin{align}\label{eq: matrix comparison all}
        \bLambda_{m,k,h} \succeq \frac{1}{1+M\alpha} \bLambda_{k,h}^{\textnormal{all}} , \ \forall \ h \in [H].
    \end{align}
\end{lemma}

\begin{proof}[Proof of Lemma~\ref{lemma: matrix comparison}]
    See Appendix~\ref{sec: proof of matrix comparison lemma}.
\end{proof}

The following two lemmas provides the concentration of self-normalized martingales in the asynchronous setting, where the first lemma applies to a fixed $V$ function, and the second one applies to the $V_{m_k,t_k,h+1}$ function in Algorithm~\ref{alg: main} via a covering argument. 
With a proper choice of $\alpha$, the bounds can be reduced to $\tilde{\mathcal{O}} (H \sqrt{d})$ and $\tilde{\mathcal{O}} (H d)$, respectively. 
These are identical to the result under the single-agent case~\citep{jin2020provably}.

\begin{lemma}\label{lemma: self norm fixed V}
    Under the setting of Theorem~\ref{thm: upper bound}, for any fixed $V \in \mathcal{V}$, with probability at least $1-\delta$, for any $k\in[K]$ and $h \in [H]$, it holds that
    \begin{align*}
        &\left\|\sum_{\tau=1, \tau\in\cD_{t_k, h}^{\textnormal{ser}}}^{t_k-1} \bphi_{\tau,h}\left[V(s_{\tau, h+1}) - \PP_h V (s_{\tau, h}, a_{\tau,h}) \right]\right\|_{\bLambda_{m_k,t_k,h}^{-1} } 
        \\ & \leq 2\left( M \sqrt{\alpha} + \sqrt{1+M\alpha}\right) \cdot H  \cdot \left( \sqrt{  \log\left(\left( \frac{K+\alpha \lambda}{\alpha \lambda }\right)^{d/2}\right) +  \log\left(\left( \frac{K+\lambda}{\lambda } \right)^{d/2}\right)+2\log \left(\frac{1}{\delta}\right) } \right).
    \end{align*}
\end{lemma}

\begin{proof}[Proof of Lemma~\ref{lemma: self norm fixed V}]
    See Appendix~\ref{sec: proof of self norm fixed V lemma}.
\end{proof}

\begin{lemma}\label{lemma: self normal term bound}
    Under the setting of Theorem~\ref{thm: upper bound},  with probability at least $1-\delta$, for any $k\in[K]$ and $h \in [H]$, it holds that
    \begin{align*}
        & \left\|\sum_{\tau=1, \tau\in\cD_{t_k, h}^{\textnormal{ser}}}^{t_k-1} \bphi_{\tau,h}\left[V_{m_k, t_k,h+1}(s_{\tau, h+1}) - \PP_h V_{m_k,t_k,h+1} (s_{\tau, h}, a_{\tau,h}) \right]\right\|_{\bLambda_{m_k,t_k,h}^{-1} } 
        \\ & \leq C(M\sqrt{\alpha}+\sqrt{1+M\alpha})H\sqrt{\iota} + C dH/\sqrt{\lambda} ,
    \end{align*}where $C$ is some universal constant and 
    \begin{align*}
        \iota \coloneqq d\log\left(\frac{K+\alpha\lambda}{\alpha\lambda}\right)+d\log\left(\frac{K+\lambda}{\lambda}\right)+\log\frac{1}{\delta}+d\log\left(2+\frac{8k^3}{\lambda}\right) + d^2\log\left(1+\frac{8\beta^2k^2}{\lambda d^{1.5}H^2}\right).
    \end{align*}
\end{lemma}

\begin{proof}[Proof of Lemma~\ref{lemma: self normal term bound}]
    See Appendix~\ref{sec: proof of self norm lemma}.
\end{proof}

% \begin{proof}[Proof of Lemma~\ref{lemma: async elliptical potential}]
%     See Appendix~\ref{sec: proof of async elliptical potential}.
% \end{proof} 

The next lemma is applied to bound the number of communication rounds of Algorithm~\ref{alg: main}. 
We first divide the episodes into different epochs.

\begin{lemma}\label{lemma: per epoch communication}
    For each $i\geq 0$, define 
    $\tilde{K}_i = \min \left\{ k \in [K] : \ \exists h \in [H] \ \textnormal{s.t.} \  \det(\bLambda_{k,h}^{\textnormal{all}}) \geq 2^i \lambda^d \right\}$. Divide all episodes into epochs where epoch i is given as $\{\tilde{K}_i, \tilde{K}_i+1,\cdots, \tilde{K}_{i+1}-1\}$, where $i \geq 0$. Then within any epoch $i$, the total number of communication rounds is upper bounded by $\mathcal{O}(H(M+1/\alpha))$. 
\end{lemma}

\begin{proof}[Proof of Lemma~\ref{lemma: per epoch communication}]
    The proof follows from the a modified argument from that of Lemma~6.2 in \citep{he2022simple}. 
    Different from \citet{he2022simple}, by Line~\ref{algline: determinant condition} of Algorithm~\ref{alg: main}, a communication is triggered if any of the $H$ determinant conditions are satisfied. As a result, the communication number is at most $H$ times the upper bound in Lemma~6.2 in \citet{he2022simple}.
\end{proof}

\subsection{Proof of Theorem~\ref{thm: upper bound}}\label{sec: detail of upper bound}

\begin{proof}[Proof of Theorem~\ref{thm: upper bound}] 
    We first prove the regret upper bound. By Lemma~\ref{lemma: optimism}, the regret can be upper bounded as
    \begin{align}\label{eq: upper bound 1}
    R(K) & = \sum_{k=1}^K \left[ V^{\star}_1 (s_{k,1}) - V_1^{\pi_{m_k,k}} (s_{k,1}) \right] \notag
    \\ & \leq \sum_{k=1}^K \left[ V_{m_k,k,1}(s_{k,1}) - V_1^{\pi_{m_k,k}} (s_{k,1}) \right] \notag 
    \\ & = \sum_{k=1}^K \left[ V_{m_k,t_k(m_k),1}(s_{k,1}) - V_1^{\pi_{m_k,k}} (s_{k,1})  \right] \notag
    \\ & \leq \sum_{k=1}^{K}\sum_{h=1}^{H} \left ( \mathbb{E}\left[ \xi_{k, h+1} \middle| s_{k,h}, a_{k,h} \right] - \xi_{k,h+1} \right) \notag 
    \\ & + \sum_{k=1}^K \min\left\{V_{m_k,t_k(m_k),1}(s_{k,1}) - V_1^{\pi_{m_k,k}} (s_{k,1}), \ \sum_{h=1}^H 2 \beta \sqrt{\bphi(s_{k,h}, a_{k,h}) \bLambda_{m_k,t_k,h}^{-1}\bphi(s_{k,h}, a_{k,h})} \right\},
    \end{align}where the first inequality is by Lemma~\ref{lemma: optimism}, and the second inequality is by Lemma~\ref{lemma: recursion}. 
    The minimum in the last step might seems odd at first, but will turn out to be necessary later. 
    Bounding the first term in the above is straightforward using martingale convergence \citep{jin2020provably}. Specifically, by the definition of in Lemma~\ref{lemma: recursion}, the first term can be written as
    \begin{align*}
        &\sum_{k=1}^{K}\sum_{h=1}^{H} \left ( \mathbb{E}\left[ \xi_{k, h+1} \middle| s_{k,h}, a_{k,h} \right] - \xi_{k,h+1} \right) 
        \\ & = \sum_{k=1}^{K}\sum_{h=1}^{H} \left ( \mathbb{E}\left[ [V_{m_k, t_k, h+1}(s_{k,h+1}) - V_{h+1}^{\pi_{m_k,k}}(s_{k,h+1})] \middle| s_{k,h+1}, a_{k,h+1} \right] - [V_{m_k, t_k, h+1}(s_{k,h+1}) - V_{h+1}^{\pi_{m_k,k}}(s_{k,h+1})] \right)
    \end{align*}Above summation can be viewed as the sum of a martingale difference sequence since $V_{m_k,t_k,h+1}$ and $V_{h+1}^{\pi_{m_k,k}}$ are independent of the observation in episode $k$. Since $|V_{m_k, t_k, h+1}(s_{k,h+1}) - V_{h+1}^{\pi_{m_k,k}}(s_{k,h+1})|\leq 2H$, by Azuma-Hoeffding inequality, with probability at least $1-\delta$, for all $k,h$, it holds that
    \begin{align}\label{eq: upper bound martingale term}
        &\sum_{k=1}^{K}\sum_{h=1}^{H} \left ( \mathbb{E}\left[ \xi_{k, h+1} \middle| s_{k,h}, a_{k,h} \right] - \xi_{k,h+1} \right) \leq 2 H^{3/2} \sqrt{K \log(2/\delta) },
    \end{align}where $\{\xi_{k,h}\}_{k,h \in [K]\times[H]}$ are defined in Lemma~\ref{lemma: recursion}.
    For the second term, note that instead of bounding $2\beta \sum_{k=1}^K \sum_{h=1}^H \sqrt{\bphi(s_{k,h}, a_{k,h}) \bLambda_{m_k,t_k,h}^{-1}\bphi(s_{k,h}, a_{k,h})}$ directly, we construct a new term involving a minimum between the bonus and the per-episode regret bound $V_{m_k,t_k(m_k),1}(s_{k,1}) - V_1^{\pi_{m_k,k}} (s_{k,1})$.  
    The reason behind this is that the sum of bonus along cannot be bounded using the standard elliptical potential argument~\citep{abbasi2011improved} due to the asynchronous nature of the communication protocol. Its analysis turns out to be much more involved and therefore is summarized  separately in Lemma~\ref{lemma: async elliptical potential}.
    Now, combining Lemma~\ref{lemma: async elliptical potential} and \eqref{eq: upper bound martingale term}, we finish the proof of the regret upper bound.

    The proof of the communication complexity of Algorithm~\ref{alg: main} is straightforward given the simple form of our determinant-based criterion. By Lemma~\ref{lemma: per epoch communication}, it remains to bound the number of epochs. 
    Recall from Assumption~\ref{assump: linear_MDP} that $\|\bphi(\cdot,\cdot)\|\leq 1$. 
    This implies that, for any $h \in [H]$, 
    \begin{align*}
        \det(\bLambda_{K,h}^{\textnormal{all}}) \leq \left(\lambda + \frac{1}{d}\sum_{k=1}^K \|\bphi_{k,h}\|_2^2 \right)^d \leq \lambda^d \left( 1+\frac{K}{\lambda d}\right)^d. 
    \end{align*} By the definition of $\tilde{K}_i$ from Lemma~\ref{lemma: per epoch communication}, in order for $\tilde{K}_i$ to be non-empty, $i$ should satisfy
    \begin{align*}
        2^i \lambda^d \leq \lambda^d \left( 1+\frac{K}{\lambda d}\right)^d,
    \end{align*}which implies $i\leq \log 2 \cdot d \log (1+K/\lambda d)$. 
    Together with Lemma~\ref{lemma: per epoch communication}, the total communication number is upper bounded by $H(M+1/\alpha) \cdot \log 2 \cdot d \log (1+K/\lambda d)$ up to some constant factor. This finishes the proof. 
\end{proof}

\section{Proof of Technical Lemmas}
\subsection{Proof of Lemma~\ref{lemma: parameter distance}}\label{sec: proof of parameter distance}

\begin{proof}[Proof of Lemma~\ref{lemma: parameter distance}]
Recall the definition of $\wb_{m,k+1,h}$ from \eqref{eq:w}, and the definition of $n_k(\cdot)$ from Appendix~\ref{sec: clarification of notation}. 
Then since $\wb_{m_k,k,h}=\wb_{m_k,t_k,h}$ is computed using all the trajectories available to agent $m_k$ by the beginning episode $t_k$, we can write
\begin{align*}
    \wb_{m_k, t_k, h} = \bLambda_{m_k,t_k,h}^{-1} \sum_{\tau=1, \tau\in\cD_{n_k(m_k), h}^{\textnormal{ser}}}^{t_k-1} \bphi_{\tau,h}\cdot\left[ r_{\tau, h} + V_{m_k, t_k,h+1}(s_{\tau, h+1}) \right],
\end{align*}where $\tau\in\cD_{n_k(m_k), h}^{\textnormal{ser}}$ denotes all the data uploaded to the server by the end of episode $n_k$. 
Note that this is well-defined since $n_k(m_k)$ is the most recent episode before $k$ when agent $m_k$ updates its policy, and therefore its local data is also included in $\cD_{n_k(m_k), h}^{\textnormal{ser}}$.
In the following we simply use $n_k$ instead of $n_k(m_k)$ since there is no confusion.  
We then write 
    \begin{align}\label{eq: proof w eq 1}
        & \wb_{m_k, t_k, h} - \wb^\pi_h \notag
        \\ & = \bLambda_{m_k,t_k,h}^{-1} \sum_{\tau=1, \tau\in\cD_{n_k, h}^{\textnormal{ser}}}^{t_k-1} \bphi_{\tau,h}\cdot\left[ r_{\tau, h} + V_{m_k, t_k,h+1}(s_{\tau, h+1}) \right] - \wb^\pi_h\notag
        \\ & = \bLambda_{m_k,t_k,h}^{-1}\left\{ -\lambda \wb^\pi_h + \sum_{\tau=1, \tau\in\cD_{n_k, h}^{\textnormal{ser}}}^{t_k-1} \bphi_{\tau,h}\left[V_{m_k, t_k,h+1}(s_{\tau, h+1}) - \PP_h V_{h+1}^\pi (s_{\tau, h}, a_{\tau,h}) \right] \right\} \notag
        \\ & = - \underbrace{\lambda \bLambda_{m_k,t_k,h}^{-1} \wb^\pi_h}_{\vb_1} + \underbrace{\bLambda_{m_k,t_k,h}^{-1} \sum_{\tau=1, \tau\in\cD_{n_k, h}^{\textnormal{ser}}}^{t_k-1} \bphi_{\tau,h}\left[V_{m_k, t_k,h+1}(s_{\tau, h+1}) - \PP_h V_{m_k,t_k,h+1} (s_{\tau, h}, a_{\tau,h}) \right]}_{\vb_2} \notag
        \\ & \qquad + \underbrace{\bLambda_{m_k,t_k,h}^{-1} \sum_{\tau=1, \tau\in\cD_{n_k, h}^{\textnormal{ser}}}^{t_k-1} \bphi_{\tau,h}\PP_h\left[V_{m_k, t_k,h+1} - V_{h+1}^\pi \right](s_{\tau, h}, a_{\tau,h})}_{\vb_3}. 
    \end{align}
For the first term, we have 
\begin{align}\label{eq: eq: proof w term 1}
    \left| \bphi(s,a)^\top \vb_1 \right| \leq \sqrt{\lambda} \left\| \wb_h^\pi\right\|_2 \sqrt{\bphi(s,a)^\top \bLambda^{-1}_{m_k, t_k, h} \bphi(s,a) } \leq 2H\sqrt{d\lambda} \sqrt{\bphi(s,a)^\top \bLambda^{-1}_{m_k, t_k, h} \bphi(s,a) },
\end{align}where the first step is by $\bLambda^{-1}_{m_k, t_k, h}\preccurlyeq \lambda^{-1} \Ib$ and the second step is by Lemma~\ref{lemma: bound norm of w pi}. For the second term, we have 
\begin{align}\label{eq: eq: proof w term 2}
    & \left| \bphi(s,a)^\top \vb_2 \right| \notag
    \\ & \leq \underbrace{\left\|\sum_{\tau=1, \tau\in\cD_{n_k, h}^{\textnormal{ser}}}^{t_k-1} \bphi_{\tau,h}\left[V_{m_k, t_k,h+1}(s_{\tau, h+1}) - \PP_h V_{m_k,t_k,h+1} (s_{\tau, h}, a_{\tau,h}) \right]\right\|_{\bLambda_{m_k,t_k,h}^{-1} }}_{\chi} \cdot \sqrt{\bphi(s,a)^\top \bLambda^{-1}_{m_k, t_k, h} \bphi(s,a)}.
\end{align}
For the third term, we have
\begin{align}\label{eq: eq: proof w term 3}
    &  \bphi(s,a)^\top \vb_3 
    \\ & = \left\langle \bphi(s,a), \bLambda_{m_k,t_k,h}^{-1} \sum_{\tau=1, \tau\in\cD_{n_k, h}^{\textnormal{ser}}}^{t_k-1} \bphi_{\tau,h}\PP_h\left[V_{m_k, t_k,h+1} - V_{h+1}^\pi \right](s_{\tau, h}, a_{\tau,h}) \right\rangle \notag
    \\ & \leq \left\langle \bphi(s,a), \bLambda_{m_k,t_k,h}^{-1} \sum_{\tau=1, \tau\in\cD_{n_k, h}^{\textnormal{ser}}}^{t_k-1} \bphi_{\tau,h}\bphi_{\tau,h}^\top\int\left[V_{m_k, t_k,h+1} - V_{h+1}^\pi \right](s') \diff \bmu_h (s') \right\rangle \notag
    \\ & \leq  \left\langle \bphi(s,a), \int\left[V_{m_k, t_k,h+1} - V_{h+1}^\pi \right](s') \diff \bmu_h (s') \right\rangle \notag
    \\ & \qquad - \lambda \left\langle \bphi(s,a), \bLambda_{m_k,t_k,h}^{-1} \int\left[V_{m_k, t_k,h+1} - V_{h+1}^\pi \right](s') \diff \bmu_h (s') \right\rangle \notag
    \\ & = \PP_h \left[V_{m_k, t_k,h+1} - V_{h+1}^\pi \right] (s,a) - \lambda \left\langle \bphi(s,a), \bLambda_{m_k,t_k,h}^{-1} \int\left[V_{m_k, t_k,h+1} - V_{h+1}^\pi \right](s') \diff \bmu_h (s') \right\rangle \notag
    \\ & \leq \PP_h \left[V_{m_k, t_k,h+1} - V_{h+1}^\pi \right] (s,a) + 2 H \sqrt{d\lambda} \cdot \sqrt{\bphi(s,a)\bLambda_{m_k, t_k, h}^{-1} \bphi(s,a)},
\end{align}where the last step holds because $\|\bmu_h\| \leq \sqrt{d}$ by Assumption~\ref{assump: linear_MDP}.
Combining \eqref{eq: proof w eq 1}, \eqref{eq: eq: proof w term 1}, \eqref{eq: eq: proof w term 2} and \eqref{eq: eq: proof w term 3}, we have 
\begin{align}\label{eq: proof w recursive 1}
    & \left| \bphi(s,a)^\top (\wb_{m_k, t_k, h} - \wb^\pi_h) -  \PP_h \left[V_{m_k, t_k,h+1} - V_{h+1}^\pi \right] (s,a)\right|\notag
    \\ & \leq \left(4H\sqrt{d\lambda} + \chi \right) \cdot \sqrt{\bphi(s,a)\bLambda_{m_k, t_k, h}^{-1} \bphi(s,a)}.
\end{align}
It remains to show that the choice of $\beta$ satisfies
\begin{align*}
    4H\sqrt{d\lambda}+\chi \leq \beta. 
\end{align*}
By Lemma~\ref{lemma: self normal term bound}, we want to show
\begin{align*}
4H\sqrt{d\lambda}+C(M\sqrt{\alpha}+\sqrt{1+M\alpha})H\sqrt{\iota} + C dH/\sqrt{\lambda} \leq \beta. 
\end{align*}
Plugging in the choice of $\beta$ and the definition of $\iota$ from Lemma~\ref{lemma: self normal term bound} and simplifying the expression, it suffices to show that there exists $c_\beta$ such that
\begin{align*}
    &C\left[\log\left(2 + \frac{K}{\delta\min\{1,\lambda,\alpha\lambda\}} \right)+ \log\left( c_\beta Hd(M\sqrt{\alpha}+\sqrt{1+M\alpha}) \right)\right]
    \\&\leq c_\beta^2 \left[  \log\left(2 + \frac{K}{\delta\min\{1,\lambda,\alpha\lambda\}} \right) + \log\left( Hd(M\sqrt{\alpha}+\sqrt{1+M\alpha}) \right) \right],
\end{align*}where $C$ is some universal constant. The existence of such $c_\beta$ is clear.
Therefore, we conclude that 
\begin{align*}
    \left| \bphi(s,a)^\top \left( \wb_{m_k,t_k,h} - \wb^\pi_h\right) - \PP_h \left[V_{m_k,t_k,h+1} - V^\pi_{h+1}\right] (s,a) \right| \leq \beta \sqrt{\bphi(s, a)^\top \bLambda^{-1}_{m_k, t_k,h} \bphi(s,a)}.
\end{align*}
\end{proof}

\subsection{Proof of Lemma~\ref{lemma: optimism}}\label{sec: proof of optimism}
\begin{proof}[Proof of Lemma~\ref{lemma: optimism}]
    The proof follows from the same induction argument in Lemma B.5 of  \citep{jin2020provably}. For completeness we introduce the proof here.
    For step $H$, by Lemma~\ref{lemma: parameter distance}, we have
    \begin{align*}
        \left| \bphi(s,a)^\top \wb_{m_k,t_k,H} - Q^{\star}_H(s,a) \right| \leq \beta \sqrt{\bphi(s, a)^\top \bLambda^{-1}_{m_k, t_k,H} \bphi(s,a)},
    \end{align*}since $V_{m_k,t_k,H+1} = V^\star_{H+1}=0$. This implies 
    \begin{align*}
        Q^\star_{H}(s,a) \leq \bphi(s,a)^\top \wb_{m_k,t_k,H} + \beta \sqrt{\bphi(s, a)^\top \bLambda^{-1}_{m_k, t_k,H} \bphi(s,a)} \leq Q_{m_k,t_k,H}(s,a).
    \end{align*}Now suppose we have proved $Q^\star_{h+1}(s,a)\leq Q_{m_k,t_k,h+1}(s,a)$. Then by Lemma~\ref{lemma: parameter distance} again, we have 
    \begin{align*}
        Q^{\star}_h(s,a)+ \PP_h \left[V_{m_k,t_k,h+1} - V^\star_{h+1}\right] (s,a) -  \bphi(s,a)^\top \wb_{m_k,t_k,h} \leq \beta \sqrt{\bphi(s, a)^\top \bLambda^{-1}_{m_k, t_k,h} \bphi(s,a)}, 
    \end{align*}which implies
    \begin{align*}
        Q^{\star}_h(s,a) & \leq \bphi(s,a)^\top \wb_{m_k,t_k,h} + \beta \sqrt{\bphi(s, a)^\top \bLambda^{-1}_{m_k, t_k,h} \bphi(s,a)} - \PP_h \left[V_{m_k,t_k,h+1} - V^\star_{h+1}\right] (s,a)
        \\ & \leq \bphi(s,a)^\top \wb_{m_k,t_k,h} + \beta \sqrt{\bphi(s, a)^\top \bLambda^{-1}_{m_k, t_k,h} \bphi(s,a)},
    \end{align*}where the last step is by the induction hypothesis that $V_{m_k,t_k,h+1} - V^\star_{h+1} \geq 0$.
    Therefore, we conclude that 
    \begin{align*}
        Q^{\star}_h(s,a) & = \min\{H-h+1, Q^{\star}_h(s,a) \} 
        \\ & \leq \min\{H-h+1, \bphi(s,a)^\top \wb_{m_k,t_k,h} + \beta \sqrt{\bphi(s, a)^\top \bLambda^{-1}_{m_k, t_k,h} \bphi(s,a)}\} = Q_{m_k,t_k,h} = Q_{m_k,k,h}.
    \end{align*}
\end{proof}

\subsection{Proof of Lemma~\ref{lemma: recursion}}\label{sec: proof of recursion}

\begin{proof}[Proof of Lemma~\ref{lemma: recursion}]
    Lemma~\ref{lemma: parameter distance} and the definition of $Q_{m_k,t_k,h}$ imply that for any $k,h$ , 
    \begin{align*}
          Q_{m_k,t_k,h}(s_{k,h},a_{k,h}) - Q_h^{\pi_{m_k,k}} (s_{k,h},a_{k,h}) & \leq \PP_h \left[V_{m_k,t_k,h+1} - V^{\pi_{m_k,k}}_{h+1}\right] (s_{k,h},a_{k,h}) + 2\beta \sqrt{\bphi_{k,h}^\top \bLambda^{-1}_{m_k, t_k,h} \bphi_{k,h}}.
    \end{align*}By the definition of $V_{m_k,t_k,h+1} $ and $ V^{\pi_{m_k,k}}_{h+1}$, we have $\xi_{k,h}=Q_{m_k,t_k,h}(s_{k,h},a_{k,h}) - Q_h^{\pi_{m_k,k}} (s_{k,h},a_{k,h})$, and it follows that 
    \begin{align*}
        \xi_{k,h} \leq \EE[\xi_{k,h+1}|s_{k,h}, a_{k,h}] + 2\beta\sqrt{\bphi(s_{k,h}, a_{k,h}) \bLambda_{m_k,t_k,h}^{-1}\bphi(s_{k,h}, a_{k,h})}.
    \end{align*}
\end{proof}

\subsection{Proof of Lemma~\ref{lemma: norm w estimate}}\label{sec: proof of norm w}
\begin{proof}[Proof of Lemma~\ref{lemma: norm w estimate}]
    The proof follows from that of Lemma B.2 in \citep{jin2020provably}. 
    Specifically, the estimated parameters $\wb_{m,k,h}$ take the same form as $\wb_{h}^k$'s in \citep{jin2020provably} if we re-index the vectors $\bphi_{k,h}$'s that are available to agent $m$ at episode~$k$.
\end{proof}

\subsection{Proof of Lemma~\ref{lemma: matrix comparison}}\label{sec: proof of matrix comparison lemma}

\begin{proof}[Proof of Lemma~\ref{lemma: matrix comparison}]
    Fix some episode $k$ and agent $m$. Recall from Appendix~\ref{sec: clarification of notation} that $N_k(m)\leq k$ is the last episode that agent $m$ participates up until the end of episode $k$. 
    If agent $m$ communicates with server in episode $N_k(m)$, then 
    \begin{align*}
        \lambda \Ib + \sum_{m'\in [M]} \bLambda_{m', k , h}^{\textnormal{up}} \succeq \mathbf{0} = \bLambda_{m,N_k(m),h}^{\textnormal{loc}} = \bLambda_{m,k,h}^{\textnormal{loc}} . 
    \end{align*}
    If agent $m$ does not participates in episode $N_k(m)$, then by Line~\ref{algline: determinant condition} of Algorithm~\ref{alg: main}, it holds that, for all $h\in[H]$, 
    \begin{align*}
        \det(\bLambda_{m,N_k(m),h}+\bLambda_{m,N_k(m), h}^{\text{loc}}) \leq (1+\alpha) \det(\bLambda_{m,N_k(m),h}).
    \end{align*}Since no further participation happens between $[N_k(m)+1, k]$, above implies
    \begin{align*}
        \det(\bLambda_{m,k,h}+\bLambda_{m,k, h}^{\text{loc}}) \leq (1+\alpha) \det(\bLambda_{m,k,h}).
    \end{align*} Applying Lemma~\ref{lemma: determinant ratio}, we get 
    \begin{align*}
        \frac{\xb^\top(\bLambda_{m,k,h}+\bLambda_{m,k, h}^{\text{loc}}) \xb}{\xb^\top \bLambda_{m,k,h} \xb} \leq \frac{\det(\bLambda_{m,k,h}+\bLambda_{m,k, h}^{\text{loc}})}{\det(\bLambda_{m,k,h})}\leq 1 + \alpha,
    \end{align*} and it follows that 
    \begin{align*}
        \xb^\top \bLambda_{m,k, h}^{\text{loc}} \xb \leq \alpha \xb^\top \bLambda_{m,k,h} \xb .
    \end{align*}Finally, we conclude that
    \begin{align*}
        \lambda \Ib + \sum_{m'\in [M]} \bLambda_{m', k , h}^{\textnormal{up}} \succeq \bLambda_{m, k, h} \succeq \frac{1}{\alpha} \bLambda_{m,k,h}^{\textnormal{loc}},
    \end{align*}where the first step follows from the fact that $\bLambda_{m,k,h}$ is downloaded at some episode $n_k(m)< k$, and the definition of $\bLambda_{m', k , h}^{\textnormal{up}}$ from \eqref{eq: def up information}. This proves \eqref{eq: matrix comparison up}.

    To show \eqref{eq: matrix comparison all}, suppose that agent $m$ communicates with the server at episode $\underbar{t}$, and is active for $k \in [\underbar{t}, \bar{t}]$. Applying \eqref{eq: matrix comparison up} for all $M$ agents and averaging, we have 
    \begin{align*}
        \lambda \Ib + \sum_{m'\in [M]} \bLambda_{m', k , h}^{\textnormal{up}} \succeq  \frac{1}{\alpha M} \sum_{m'\in[M]} \bLambda_{m', k , h}^{\textnormal{loc}},
    \end{align*} and it follows that, for $k \in [\underbar{t}+1, \bar{t}]$, 
    \begin{align*}
        \bLambda_{m,k,h} & = \lambda \Ib + \sum_{m'\in [M]} \bLambda_{m', \underbar{t}+1 , h}^{\textnormal{up}} 
        \\ & =  \lambda \Ib + \sum_{m'\in [M]} \bLambda_{m', k , h}^{\textnormal{up}}
        \\ & \succeq \frac{1}{1+\alpha M} \left( \lambda \Ib + \sum_{m'\in [M]} \bLambda_{m', k , h}^{\textnormal{up}} + \sum_{m'\in[M]} \bLambda_{m', k , h}^{\textnormal{loc}} \right)
        \\ & = \frac{1}{1+\alpha M} \bLambda_{k,h}^{\textnormal{all}}.
    \end{align*} Here the first step follows from the definition of $\bLambda_{m', \underbar{t}+1 , h}^{\textnormal{up}}$ from \eqref{eq: def up information} and the assumption that agent $m$ communicates with the server in episode $\underbar{t}$. The second step holds because agent $m$ is the only active agent between $[\underbar{t}, \bar{t}]$ and thus no further upload has been made by any agent during episodes $[\underbar{t}+1, \bar{t}]$. The last step follows from the definition of $\bLambda_{k,h}^{\textnormal{all}}$ and the assumption that agent $m$ is the only active agent between $[\underbar{t}, \bar{t}]$ (i.e. no other agent can upload during $[\underbar{t}, \bar{t}]$). This finishes the proof of \eqref{eq: matrix comparison all} and that of Lemma~\ref{lemma: matrix comparison}.
\end{proof}

\subsection{Proof of Lemma~\ref{lemma: self norm fixed V}}\label{sec: proof of self norm fixed V lemma}
\begin{proof}[Proof of Lemma~\ref{lemma: self norm fixed V}]
    Define $\eta_{\tau,h} = V(s_{\tau, h+1}) - \PP_h V (s_{\tau, h}, a_{\tau,h}) $, and 
    \begin{align}\label{eq: proof self norm fixed def u}
        \ub_{k,h}^{\textnormal{up}} (m') & = \sum_{\tau=1, \tau\in \cD_{k,h}^{\textnormal{ser}},m_\tau = m'}^{k} \bphi_{\tau,h}\eta_{\tau,h} , \notag
        \\ \ub_{k,h}^{\textnormal{loc}}(m') & = \sum_{\tau=1, \tau\notin \cD_{k,h}^{\textnormal{ser}},m_\tau = m'}^{k} \bphi_{\tau,h}\eta_{\tau,h} ,
    \end{align}for all $m'\in[M]$ and $k,h \in [K]\times[H]$.
    Then we have 
    \begin{align}\label{eq: proof self norm fixed 0}
        & \left\|\sum_{\tau=1, \tau\in\cD_{n_k, h}^{\textnormal{ser}}}^{t_k-1} \bphi_{\tau,h}\left[V(s_{\tau, h+1}) - \PP_h V (s_{\tau, h}, a_{\tau,h}) \right]\right\|_{\bLambda_{m_k,t_k,h}^{-1} }\notag
        \\ & = \left\|\sum_{\tau=1, \tau\in\cD_{n_k, h}^{\textnormal{ser}}}^{t_k-1} \bphi_{\tau,h}\eta_{\tau, h}\right\|_{\bLambda_{m_k,t_k,h}^{-1} } \notag
        \\ & = \left\| \sum_{m'=1}^M \ub_{n_k,h}^{\textnormal{up}}(m') \right\|_{\bLambda_{m_k,t_k,h}^{-1}} \notag
        \\ & = \left\| \sum_{m'=1}^M \ub_{n_k,h}^{\textnormal{up}}(m') + \sum_{m'=1}^M \ub_{n_k,h}^{\textnormal{loc}}(m') - \sum_{m'=1}^M \ub_{n_k,h}^{\textnormal{loc}}(m')\right\|_{\bLambda_{m_k,t_k,h}^{-1}} \notag
        \\ &\leq \underbrace{\left\| \sum_{m'=1}^M \ub_{n_k,h}^{\textnormal{up}}(m') + \sum_{m'=1}^M \ub_{n_k,h}^{\textnormal{loc}}(m')\right\|_{\bLambda_{m_k,t_k,h}^{-1}}}_{(\textnormal{I})} + \underbrace{\left\| \sum_{m'=1}^M \ub_{n_k,h}^{\textnormal{loc}}(m')\right\|_{\bLambda_{m_k,t_k,h}^{-1}}}_{(\textnormal{II})} ,
    \end{align}where the first and the second steps are by the definition of $\eta_{\tau,h}$ and $\ub_{n_k,h}^{\textnormal{up}}$, and the last step is by triangle inequality. 

    For $(\textnormal{I})$, we have that with probability at least $1-\delta$, for all $k$, 
    \begin{align}\label{eq: proof self norm fixed 1}
        (\textnormal{I}) & = \left\| \sum_{\tau=1}^{n_k} \bphi_{\tau,h} \eta_{\tau,h}\right\|_{\bLambda_{m_k,t_k,h}^{-1}} \notag
        \\ & = \left\| \sum_{\tau=1}^{n_k} \bphi_{\tau,h} \eta_{\tau,h}\right\|_{\bLambda_{m_k,n_k+1,h}^{-1}} \notag
        \\ & \leq \sqrt{1+\alpha M} \left\| \sum_{\tau=1}^{n_k} \bphi_{\tau,h} \eta_{\tau,h}\right\|_{(\bLambda_{n_k+1,h}^{\textnormal{all}})^{-1}}\notag 
        \\ & \leq \sqrt{1+\alpha M} \cdot \sqrt{4H^2 \left[ \log\left(\left( \frac{K+\lambda}{\lambda } \right)^{d/2}\right)+\log \left(\frac{1}{\delta}\right)\right]}.
    \end{align}Here the first inequality is given by \eqref{eq: matrix comparison all} in Lemma~\ref{lemma: matrix comparison}. The second inequality is derived by applying Theorem~\ref{thm: hoeffding self normalized} with $|\eta_{\tau,h}|\leq 2H$ (according to Line~\ref{algline: update Q}), and 
    \begin{align*}
        \det(\bLambda_{n_k+1,h}^{\textnormal{all}})\leq (\| \bLambda_{n_k+1,h}^{\textnormal{all}}\|_2)^{d} \leq \left\|\sum_{\tau=1}^{n_k}\bphi_{\tau,h} \bphi_{\tau,h}^\top + \lambda \Ib \right\|_2 \leq (K + \lambda)^d.
    \end{align*}
    For $(\textnormal{II})$, first note that for any $m \in [M]$, Lemma~\ref{lemma: matrix comparison} implies 
    \begin{align*}
        \bLambda_{m_k,t_k,h} & \succeq \lambda \Ib + \sum_{m'=1}^M \bLambda_{m',n_k(m_k),h}^{\textnormal{up}} \geq \frac{1}{\alpha} \bLambda_{m,n_k(m_k),h}^{\textnormal{loc}}.
    \end{align*}It follows that for any $m'\in [M]$,
    \begin{align*}
        \bLambda_{m_k,t_k,h} & \succeq \frac{\lambda \Ib}{2} + \frac{1}{2\alpha} \bLambda_{m',n_k(m_k),h}^{\textnormal{loc}} = \frac{1}{2\alpha} \left( \alpha \lambda \Ib + \bLambda_{m',n_k(m_k),h}^{\textnormal{loc}} \right), 
    \end{align*}and thus 
    \begin{align*}
        \left\|\ub_{n_k,h}^{\textnormal{loc}}(m') \right\|_{\bLambda_{m_k,t_k,h}^{-1}} & \leq  \left\|\ub_{n_k,h}^{\textnormal{loc}}(m') \right\|_{\frac{1}{2\alpha} \left( \alpha \lambda \Ib + \bLambda_{m',n_k(m_k),h}^{\textnormal{loc}} \right)^{-1}} 
        \\ & = \sqrt{2\alpha} \left\|\ub_{n_k,h}^{\textnormal{loc}}(m') \right\|_{ \left( \alpha \lambda \Ib + \bLambda_{m',n_k(m_k),h}^{\textnormal{loc}} \right)^{-1}} 
        \\ & \leq \sqrt{2\alpha} \sqrt{4H^2 \left[ \log\left(\left( \frac{K+\alpha \lambda}{\alpha \lambda }\right)^{d/2}\right) + \log\left( \frac{1}{\delta}\right) \right]},
    \end{align*}where the last step holds by Theorem~\ref{thm: hoeffding self normalized}, the definition of $\ub_{n_k,h}^{\textnormal{loc}}(m')$ from \eqref{eq: proof self norm fixed def u}, and the definition of $\bLambda_{m',n_k(m_k),h}^{\textnormal{loc}}$ from \eqref{eq: loc information}. 
    We then conclude that 
    \begin{align}\label{eq: proof self norm fixed 2}
        (\textnormal{II}) & \leq 2 M \sqrt{\alpha} H \sqrt{  \log\left(\left( \frac{K+\alpha \lambda}{\alpha \lambda }\right)^{d/2}\right) + \log\left( \frac{1}{\delta}\right) } .
    \end{align}
    Combining \eqref{eq: proof self norm fixed 0}, \eqref{eq: proof self norm fixed 1} and \eqref{eq: proof self norm fixed 2}, we conclude that
    \begin{align*}
        & \left\|\sum_{\tau=1, \tau\in\cD_{n_k, h}^{\textnormal{ser}}}^{t_k-1} \bphi_{\tau,h}\left[V(s_{\tau, h+1}) - \PP_h V (s_{\tau, h}, a_{\tau,h}) \right]\right\|_{\bLambda_{m_k,t_k,h}^{-1} } 
        \\ & \leq 2\left( M \sqrt{\alpha} + \sqrt{1+M\alpha}\right) \cdot H  \cdot \left( \sqrt{ \log\left(\left( \frac{K+\alpha \lambda}{\alpha \lambda }\right)^{d/2}\right) +  \log\left(\left( \frac{K+\lambda}{\lambda } \right)^{d/2}\right)+2\log \left(\frac{1}{\delta}\right) } \right) . 
    \end{align*}
\end{proof}

\subsection{Proof of Lemma~\ref{lemma: self normal term bound}}\label{sec: proof of self norm lemma}
    With Lemma~\ref{lemma: self norm fixed V} established, 
    the proof of Lemma~\ref{lemma: self normal term bound} relies on the classical $\ell_\infty$ covering net argument of the linear MDPs, developed by Lemma B.3 in \citet{jin2020provably}. 

    \begin{lemma}[Lemma~D.6, \citealt{jin2020provably}]\label{lemma: covering number}
        Let $\mathcal{V}$ denote a class of functions from $\mathcal{S}$ to $\mathbb{R}$ such that each $V \in \mathcal{V}$ can be parametrized as
    \begin{align*}
        V(\cdot) &= \max_{a\in\mathcal{A}} \left[ \bphi(\cdot,\cdot)^{\top} \wb + \beta \cdot \sqrt{\bphi(\cdot,\cdot)^\top \bLambda^{-1} \bphi(\cdot,\cdot)}  \right]_{[0, H-h+1]}, 
    \end{align*}where the parameters $(\wb, \bLambda, \beta)$ satisfy $\|\wb\|\leq W $, $0\leq \beta \leq B $, and $\bLambda \succeq \lambda \Ib$ for some $\lambda >0$. Suppose $\|\bphi(\cdot,\cdot)\| \leq 1$. The $\epsilon$-covering number of $\mathcal{V}$ with respect to the $\ell_\infty$ norm satisfies
    \begin{align*}
        \log(\mathcal{N}_{\epsilon}) \leq d \log(1+4W/\epsilon) + d^2 \log[1+8d^{1/2}B^2/(\lambda \epsilon^2)]. 
    \end{align*}
    \end{lemma}

We can now prove Lemma~\ref{lemma: self normal term bound} using Lemma~\ref{lemma: self norm fixed V} and Lemma~\ref{lemma: covering number}.

\begin{proof}[Proof of Lemma~\ref{lemma: self normal term bound}]
        We first fix an $\epsilon$-net of $\mathcal{V}$. 
        For each $V\in\mathcal{V}$, there exists some $\tilde{V}$ in the $\epsilon$-net, such that $\|V-\Tilde{V}\|_{\infty}\leq \epsilon$. 
        Applying a union bound over the $\epsilon$-net and Lemma~\ref{lemma: self norm fixed V}, we get that with probability at least $1-\delta$, for each $V \in \mathcal{V}$,
        \begin{align*}
        &\left\|\sum_{\tau=1, \tau\in\cD_{t_k, h}^{\textnormal{ser}}}^{t_k-1} \bphi_{\tau,h}\left[V(s_{\tau, h+1}) - \PP_h V (s_{\tau, h}, a_{\tau,h}) \right]\right\|_{\bLambda_{m_k,t_k,h}^{-1} }^2
        \\ & \leq 8\left( M \sqrt{\alpha} + \sqrt{1+M\alpha}\right)^2 \cdot H^2 \cdot \left( \log\left(\left( \frac{K+\alpha \lambda}{\alpha \lambda }\right)^{d/2}\right) +  \log\left(\left( \frac{K+\lambda}{\lambda } \right)^{d/2}\right)+2\log \left(\frac{N_{\epsilon}}{\delta}\right) \right) 
        \\ & \qquad+ 2 \left\|\sum_{\tau=1, \tau\in\cD_{t_k, h}^{\textnormal{ser}}}^{t_k-1} \bphi_{\tau,h}\left[\Delta V(s_{\tau, h+1}) - \PP_h \Delta V (s_{\tau, h}, a_{\tau,h}) \right]\right\|_{\bLambda_{m_k,t_k,h}^{-1} }^2
        \\ & \leq 8\left( M \sqrt{\alpha} + \sqrt{1+M\alpha}\right)^2 \cdot H^2 \cdot \left( \log\left(\left( \frac{K+\alpha \lambda}{\alpha \lambda }\right)^{d/2}\right) +  \log\left(\left( \frac{K+\lambda}{\lambda } \right)^{d/2}\right)+2\log \left(\frac{N_{\epsilon}}{\delta}\right) \right) 
        \\ & \qquad + \frac{8 k^2 \epsilon^2}{\lambda},
        \end{align*}where the last step follows from $\|\Delta V\|_{\infty} = \| V-\tilde{V}\|_{\infty}\leq \epsilon$ and $\bLambda_{m_k,t_k,h}\succeq \lambda \Ib$. 
        Finally, by Lemma~\ref{lemma: norm w estimate}, we have $\|\wb_{m,k,h}\| \leq 2H\sqrt{dk/\lambda}$ .  Plugging in the bound for $\mathcal{N}_{\epsilon}$ from Lemma~\ref{lemma: covering number}, we get 
        \begin{align*}
        &\left\|\sum_{\tau=1, \tau\in\cD_{t_k, h}^{\textnormal{ser}}}^{t_k-1} \bphi_{\tau,h}\left[V(s_{\tau, h+1}) - \PP_h V (s_{\tau, h}, a_{\tau,h}) \right]\right\|_{\bLambda_{m_k,t_k,h}^{-1} }^2 \leq 
        8\left( M \sqrt{\alpha} + \sqrt{1+M\alpha}\right)^2 \cdot H^2 \cdot \iota' + \frac{8k^2\epsilon^2}{\lambda},
        \end{align*}where
        \begin{align*}
            \iota' \coloneqq \frac{d}{2}\log\left(\frac{K+\alpha \lambda}{\alpha \lambda}\right) + \frac{d}{2} \log\left( \frac{K+\lambda}{\lambda}\right) + 2 \log\frac{1}{\delta} + d\log\left(2+\frac{8H^2 d k}{\lambda \epsilon^2} \right) + 2d^2\log\left( 1 + \frac{8\sqrt{d}\beta^2}{\lambda \epsilon^2}\right).
        \end{align*}
        We choose $\epsilon = dH/k$, and conclude that
        \begin{align*}
            &\left\|\sum_{\tau=1, \tau\in\cD_{t_k, h}^{\textnormal{ser}}}^{t_k-1} \bphi_{\tau,h}\left[V(s_{\tau, h+1}) - \PP_h V (s_{\tau, h}, a_{\tau,h}) \right]\right\|_{\bLambda_{m_k,t_k,h}^{-1} }
            \\ & \leq C(M\sqrt{\alpha}+\sqrt{1+M\alpha})H\sqrt{\iota} + C dH/\sqrt{\lambda},
        \end{align*}where
        \begin{align*}
            \iota = d\log\left(\frac{K+\alpha\lambda}{\alpha\lambda}\right)+d\log\left(\frac{K+\lambda}{\lambda}\right)+\log\frac{1}{\delta}+d\log\left(2+\frac{8k^3}{\lambda}\right) + d^2\log\left(1+\frac{8\beta^2k^2}{\lambda d^{1.5}H^2}\right).
        \end{align*}
        
\end{proof}

\subsection{Proof of Lemma~\ref{lemma: async elliptical potential}}\label{sec: proof of async elliptical potential}

\begin{lemma}[Repeat of Lemma~\ref{lemma: async elliptical potential}]
    Under the same assumption of Theorem~\ref{thm: upper bound}, it holds that 
    \begin{align*}
        & \sum_{k=1}^K \min\left\{V_{m_k,t_k(m_k),1}(s_{k,1}) - V_1^{\pi_{m_k,k}} (s_{k,1}), \sum_{h=1}^H 2 \beta \sqrt{\bphi(s_{k,h}, a_{k,h}) \bLambda_{m_k,t_k,h}^{-1}\bphi(s_{k,h}, a_{k,h})} \right\}
        \\ & \leq \mathcal{O}(\beta\sqrt{1+M\alpha} H\sqrt{dK \log\left(2dK/\min\{1,\lambda\}\delta\right)}).
    \end{align*} 
\end{lemma}

\begin{proof}[Proof of Lemma~\ref{lemma: async elliptical potential}]
    Suppose that agents communicate with the server at episodes $0 = k_0 < k_1 < \cdots < k_N = K+1$. Here $k_0=0$ and $k_N = K+1$ are imaginary episodes created for notational convenience. The first step is to use a reordering trick to argue that it suffices to consider the case where there is only one active agent in the episodes $[t_i, t_{i+1}-1]$. That is, $m_{t_i} = m_{t_i+1} = \cdots = m_{t_{i+1}-1}$. 
    
    To see why this is the case, suppose an agent $m$ communicates with the server at some episode $k_1$ and $k_2$. 
    Then the order of actions between $k_1$ and $k_2$ will not affect agent $m$'s covariance matrix and dataset at episode $k_1$ or $k_2$, and thus will not affect the estimated Q-function updated at the end of eposide $k_1$ and $k_2$. 
    Furthermore, agent $m$'s participation in those episodes between $[k_1+1, k_2-1]$ will also not affect the other agents' estimated Q-functions since agent $m$ does not upload any new trajectory.  
    Given the above rationale, we can reorder all the episodes in a way such that each agent communicates with the server and keeps participating until the next agent kicks in to communicates with the server.
    Note that the fundamental reason is that each agent only performs local data update between two communications, which does not affect any other agents.
    Consequently, this reordering is always valid under the current communication protocol of Algorithm~\ref{alg: main}.

    From above, in the following we only consider the case where the communication episodes are $0 = k_0 < k_1 < \cdots < k_N = K+1$, and $m_{t_i} = m_{t_i+1} = \cdots = m_{t_{i+1}-1}$ for each $i=0,\cdots, N-1$.  
    The summation of bonus can thus be rephrased as
    \begin{align}\label{eq: elliptic 2}
        & \sum_{k=1}^K \min\left\{V_{m_k,t_k(m_k),1}(s_{k,1}) - V_1^{\pi_{m_k,k}} (s_{k,1}), \sum_{h=1}^H 2 \beta \sqrt{\bphi(s_{k,h}, a_{k,h}) \bLambda_{m_k,t_k,h}^{-1}\bphi(s_{k,h}, a_{k,h})} \right\} \notag 
        \\ & \leq  2 \beta \underbrace{\sum_{i=0}^{N-1}\sum_{k=k_i+1}^{k_{i+1}-1} \sum_{h=1}^H \sqrt{\bphi_{k,h} \bLambda_{m_k,t_k,h}^{-1}\bphi_{k,h}}}_{\textnormal{I}}  \notag 
        \\ & + 2\beta\underbrace{\sum_{i=1}^{N-1} \min\left\{V_{m_{k_i},t_{k_i},1}(s_{k_i,1}) - V_1^{\pi_{m_{k_i},k_i}} (s_{k_i,1}), \sum_{h=1}^H \sqrt{\bphi_{k_i,h} \bLambda_{m_{k_i},t_{k_i},h}^{-1}\bphi_{k_i,h}}\right\}}_{\textnormal{II}}.
    \end{align}
    To bound $\textnormal{I}$, by \eqref{eq: matrix comparison all} in Lemma~\ref{lemma: matrix comparison} it holds that
    \begin{align}\label{eq: elliptic i}
        \textnormal{I} &\leq \sum_{i=0}^{N-1}\sum_{k=k_i+1}^{k_{i+1}-1} \sum_{h=1}^H \sqrt{1+M\alpha}\left\|\bphi_{k,h}\right\|_{(\bLambda_{k,h}^{\textnormal{all}})^{-1}} \leq \sqrt{1+M\alpha}\sum_{k=1}^K \sum_{h=1}^H \left\|\bphi_{k,h}\right\|_{(\bLambda_{k,h}^{\textnormal{all}})^{-1}} .
    \end{align}
    To bound $\textnormal{II}$, we apply a refined analysis tailored from \citep{he2022simple}. Specifically, we define the following indices of episode
    \begin{align*}
        \tilde{K}_{i} \coloneqq \min\left\{ k \in[K]: \ \exists h\in[H] \ \textnormal{s.t.} \ \det(\bLambda_{k,h}^{\textnormal{all}}) \geq 2^i \lambda^d \right\},
    \end{align*} and define $N'$ to be the largest integer such that $\tilde{K}_{N'}$ is non-empty. 
    For each interval $[\tilde{K}_{i}, \tilde{K}_{i+1})$, consider an arbitrary agent $m \in [M]$. 
    Suppose that during this interval agent $m$ communicates with the server at episodes $k_{i,1} < k_{i, 2} < \cdots < k_{i, z}$. Note that here we assume there are at least two communication rounds for $m$. 
    The case of 0 and 1 communication round is quite straightforward, as will be shown soon. 
    Now, for $j = 2, \cdots, z$, agent $m$ is active at episode $k_{i,j-1}$ and $k_{i,j}$. As a result, we can apply \eqref{eq: matrix comparison all} in Lemma~\ref{lemma: matrix comparison} with our reordering trick, and get that 
    \begin{align*}
        \sum_{h=1}^H \left\|\bphi_{k_{i,j},h}\right\|_{\bLambda_{m,{k_{i,j}},h}^{-1}} & \leq \sum_{h=1}^H \left\|\bphi_{k_{i,j},h}\right\|_{\bLambda_{m,{k_{i,j-1}+1},h}^{-1}} \leq \sqrt{1+M\alpha} \sum_{h=1}^H \left\|\bphi_{k_{i,j},h}\right\|_{(\bLambda_{{k_{i,j-1}+1},h}^{\textnormal{all}})^{-1}}, 
    \end{align*}where the first step is by $\bLambda_{m,{k_{i,j}},h}^{-1}  \preceq \bLambda_{m,{k_{i,j-1}+1},h}^{-1}$. 
    Furthermore, by the definition of $\tilde{K}_i$, it holds that 
    $\det(\bLambda_{\tilde{K}_{i+1}-1,h}^{\textnormal{all}}) / \det(\bLambda_{{k_{i,j-1}+1},h}^{\textnormal{all}}) \leq 2$, which implies
    \begin{align}\label{eq: elliptical 3}
        \sum_{h=1}^H \left\|\bphi_{k_{i,j},h}\right\|_{\bLambda_{m,{k_{i,j}},h}^{-1}} & \leq \sqrt{2}\sqrt{1+M\alpha} \sum_{h=1}^H \left\|\bphi_{k_{i,j},h}\right\|_{(\bLambda_{\tilde{K}_{i+1}-1,h}^{\textnormal{all}})^{-1}} \leq \sqrt{2}\sqrt{1+M\alpha} \sum_{h=1}^H \left\|\bphi_{k_{i,j},h}\right\|_{(\bLambda_{k_{i,j},h}^{\textnormal{all}})^{-1}} . 
    \end{align}
    The second step in the above holds since $k_{i,j} \leq  \tilde{K}_{i+1}-1$. 
    For those episodes $k_{i,1}$ (i.e. $j=1$), we can trivially bound the term as $\max[V_{m_{k_i},t_{k_i},1}(\cdot) - V_1^{\pi_{m_{k_i},k_i}} (\cdot)]\leq 2H$. Together with \eqref{eq: elliptic 2} and \eqref{eq: elliptical 3}, we have
    \begin{align*}
        & \sum_{k=1}^K \min\left\{V_{m_k,t_k(m_k),1}(s_{k,1}) - V_1^{\pi_{m_k,k}} (s_{k,1}), \sum_{h=1}^H 2 \beta \sqrt{\bphi(s_{k,h}, a_{k,h}) \bLambda_{m_k,t_k,h}^{-1}\bphi(s_{k,h}, a_{k,h})} \right\} \notag 
        \\ & \leq  2H N' + 2\beta \sqrt{2(1+M\alpha)} \sum_{i=0}^{N-1} \sum_{k=k_i}^{k_{i+1}-1} \sum_{h=1}^H \left\|\bphi_{k,h}\right\|_{(\bLambda_{k,h}^{\textnormal{all}})^{-1}}
        \\ & \leq 2 H N' + 4 \beta\sqrt{1+M\alpha} H\sqrt{dK \log\left(2dK/(\min\{1,\lambda\}\delta)\right)}, 
    \end{align*}where the last step follows from the standard elliptical potential argument \citep{abbasi2011improved, jin2020provably}. 
    To bound $N'$, by Assumption~\ref{assump: linear_MDP}, it holds that
    \begin{align*}
        \det(\bLambda_{k,h}^{\textnormal{all}}) \leq (\lambda + K)^d,
    \end{align*} and therefore $N' \leq dH\log(1 + K/\lambda) $. This finishes the proof.
\end{proof}

\section{Lower bound}
To prove the lower bound, we construct a series of hard-to-learn MDPs as follows. For each hard-to-learn MDP, the state space $\cS$ consists of $d/2$ different states $\cS=\{s_1,...,s_{d/2-2},g_0,g_1\}$, where $\{s_1,...,s_{d/2-2}\}$ are possible initial states and $\{g_0,g_1\}$ are absorbing states. The action space $\cA$ only consists of two different action $\{a_0,a_1\}$. For each stage, $h\in[H]$, the agent will always receive reward $1$ at state $g_0$ and reward $0$ at other states. For the stochastic transition process, the initial state $s_i$ will transit to the absorbing states $g_0$ or $g_1$, and stay at the absorbing state later. Since the state and action spaces are finite, these hard-to-learn tabular MDPs can be further represented as linear MDPs with dimension $|\cS|\times |\cA|=d$. 

Now, for each initial state $s_i$, the selection of action $a\in\{a_0,a_1\}$ can be seemed as a $2$-armed bandits problem with  Bernoulli reward ($0$ for absorbing stage $g_1$ and $H-1$ for absorbing stage $g_0$)
and we have the following Lemmas: 
\begin{lemma}[Theorem 9.1 in \citealt{lattimore2020bandit}]\label{lemma:1-up}
For any $2$-armed Bernoulli bandits problem, there exist an algorithm (e.g., MOSS algorithm in Section 9.1 of \citet{lattimore2020bandit}) with expected regret
$\EE[\text{Regret}(T)]\leq 40\sqrt{T}$.
\end{lemma}
The original Theorem 9.1 holds for multi-armed bandit with sub-Gaussian noise and we only need the results for $2$-armed Bernoulli bandits.
\begin{lemma}[Lemma D.2 in \citealt{wang2019distributed}]\label{lemma:1-low}
For any algorithm \textbf{Alg} and $T$, there exist a $2$-armed Bernoulli bandits such that the regret is lower bounded by
$\EE[\text{Regret}(T)]\ge \sqrt{T}/75$.
\end{lemma}
The lemma in \citet{wang2019distributed} extended the result for Gaussian bandit \citep{lattimore2020bandit} to Bernoulli bandits and holds for general multi-arm bandit problem. In this lower bound, we only need the results for $2$-armed bandits.

Now, we start to prove the Theorem~\ref{thm:lower_bound}, which is an extension of the lower bound results in \citet[Theorem 2]{wang2019distributed} and \citet[Theorem 5.3]{he2022simple} from bandits to MDPs.
\begin{proof}[Proof of Theorem~\ref{thm:lower_bound}]
Now, we divide the $K$ episodes to $d/2$ different epochs. For each  epoch i (from episodes $2(i-1)K/d+1$ to episode $2iK/d$), we set the initial state as $s_i$ and letting each agent $m\in[M]$ be active for $2K/(dM)$ different rounds (where we assume $2K/(dM)$ is an integer for simplicity). Now, we start to analyse the regret $\EE[\text{Regret}_{i,\textbf{Alg}_i}]$ for each epoch $i$.

For each epoch $i$ and any algorithm \textbf{Alg} for multi-agent Reinforcement Learning, we construct the auxiliary 
$\textbf{Alg}_i$ as follows: For each agent $m\in[M]$, it performs \textbf{Alg} 
until there is a communication between the agent $m$ and the server after the epoch $i-1$. 
After the communication after epoch $i-1$, the agent $m$ remove all previous information and 
perform the used Algorithm in Lemma \ref{lemma:1-up}(e.g., MOSS algorithm in \citet{lattimore2020bandit}).

In this case, for each agent $m\in[M]$, $\textbf{Alg}_i$ can only communicate with the server before epoch $i$, which can only provide information about previous states $s_1,..,s_{i-1}$. Since the agent can not receive any 
information for state $s_i$ from other agents, the performance of $\textbf{Alg}_i$ in epoch $i$ will reduce to a single agent bandit algorithm.

Now, we consider the hard-to-learn Bernoulli bandits in Lemma \ref{lemma:1-low} with rounds $T=2K/(dM)$.
Since $\textbf{Alg}_i$ reduces to a single agent bandit algorithm with Bernoulli reward ($0$ or $H-1$), 
Lemma \ref{lemma:1-low} implies that the expected regret for agent $m$ with $\textbf{Alg}_i$ 
is lower bounded by
\begin{align}
    \EE[\text{Regret}_{i,m,\textbf{Alg}_i}]\ge (H-1)\sqrt{T}/75.\label{eq:0001}
\end{align}
Taking the sum of \eqref{eq:0001} over all agents $m\in[M]$, we obtain
\begin{align}
    \EE[\text{Regret}_{i,\textbf{Alg}_i}]
    &=\sum_{m=1}^M \EE[\text{Regret}_{i,m,\textbf{Alg}_i}]\ge M(H-1)\sqrt{T}/75.\label{eq:0003}
\end{align}

For each agent $m\in[M]$, let $\delta_{i,m}$ denote the probability that agent $m$ 
will communicate with the server during epoch $i$. 
Notice that before the communication, $\textbf{Alg}_i$ has the same performance as the original 
\textbf{Alg} and the corresponding regret of $\textbf{Alg}_i$ is upper bounded by $\EE[\text{Regret}_{i,m,\textbf{Alg}}]$. After the communication during epoch $i$, $\textbf{Alg}_i$ perform the near optimal algorithm in Lemma \ref{lemma:1-up} and provides a $40(H-1)\sqrt{T}$ regret guarantee. Combining these results, the expected regret for agent $m$ with $\textbf{Alg}_i$ is upper bounded by
\begin{align}
    \EE[\text{Regret}_{i,m,\textbf{Alg}_i}]
    \leq \EE[\text{Regret}_{i,m,\textbf{Alg}}]
    +  40\delta_{i,m}(H-1)\sqrt{T}.\label{eq:0002}
\end{align}
Taking the sum of \eqref{eq:0002} over all agents $m\in[M]$, we obtain
\begin{align}
\EE[\text{Regret}_{i,\textbf{Alg}_i}]&=\sum_{m=1}^M\EE[\text{Regret}_{i,m,\textbf{Alg}_i}]\notag\\
    &\leq  \sum_{m=1}^M \EE[\text{Regret}_{i,m,\textbf{Alg}}]
    + \bigg(\sum_{m=1}^M \delta_m\bigg)  40\delta_{i,m} (H-1)\sqrt{T}\notag\\
    &= \EE[\text{Regret}_{i,\textbf{Alg}}] + 40\delta_{i} (H-1)\sqrt{T},\label{eq:0004}
\end{align}
where $\delta_i=\sum_{m=1}^M \delta_{i,m}$ is the expected communication complexity during epoch $i$. For the regret bounds in \eqref{eq:0003} and \eqref{eq:0004}, after taking an summation over all epoch $i\in[d/2]$, we have
\begin{align*}
\sum_{i=1}^{d/2}\EE[\text{Regret}_{i,m,\textbf{Alg}_i}]&\ge dM(H-1)\sqrt{T}/150,\notag\\
\sum_{i=1}^{d/2}\EE[\text{Regret}_{i,m,\textbf{Alg}_i}]&\leq  \sum_{i=1}^{d/2}\EE[\text{Regret}_{i,\textbf{Alg}}] + 40\delta_{i}(H-1) \sqrt{T} =\EE[\text{Regret}_{\textbf{Alg}}]+40\delta(H-1)\sqrt{T},
\end{align*}
where $\delta=\sum_{i=1}^{d/2} \delta_{i,m}$ denotes the expected communication complexity.
Combining these results, for any algorithm \textbf{Alg} 
with expected communication complexity $\delta\leq dM/12000=O(dM)$, we have
\begin{align*}
    \EE[\text{Regret}_{\textbf{Alg}}]&\ge dM(H-1)\sqrt{T}- 40\delta(H-1)\sqrt{T} 
    \ge dM(H-1)\sqrt{T}/2 = \Omega (H\sqrt{dMK}).
\end{align*}
This finishes the proof of Theorem~\ref{thm:lower_bound}.
\end{proof}

\section{Auxiliary Lemmas}

\begin{lemma}[Lemma D.1 in \citealt{jin2020provably}]
    Let $\bLambda_t = \lambda \Ib + \sum_{\tau=1}^t \bphi_\tau \bphi_\tau^\top$ where $\bphi_t \in \RR^d$ for all $\tau$, and $\lambda > 0$. Then 
    \begin{align*}
        \sum_{\tau = 1}^t \bphi_\tau^\top \bLambda_t^{-1} \bphi_\tau \leq d. 
    \end{align*}
\end{lemma}

\begin{lemma}[Lemma 12 in \citet{abbasi2011improved}]\label{lemma: determinant ratio}
Suppose $\Ab, \Bb\in \RR^{d \times d}$ are positive definite matrices such that $\Ab \succeq \Bb$. Then for any $\xb \in \RR^d$, $\|\xb\|_{\Ab} \leq \|\xb\|_{\Bb}\cdot \sqrt{\det(\Ab)/\det(\Bb)}$.
\end{lemma}

\subsection{Concentration Inequalities}

\begin{theorem}[Hoeffding-type inequality for self-normalized martingales \citep{abbasi2011improved}]\label{thm: hoeffding self normalized}
    Let $\{\eta_t\}_{t=1}^\infty$ be a real-valued stochastic process. Let $\{\cF_t \}_{t=0}^\infty$ be a filtration, such that $\eta_t$ is $\cF_t$-measurable. Assume $\eta_t \mid \cF_{t-1}$ is zero-mean and $R$-subgaussian for some $R > 0$, i.e.,
\begin{align*}
    \forall \lambda \in \RR, \quad \EE\left[ e^{\lambda \eta_t \mid \cF_{t-1}} \right] \leq e^{\lambda^2 R^2 /2}.
\end{align*}
Let $\{\bphi_t\}_{t=1}^\infty$ be an $\RR^d$-valued stochastic process where $\bphi_t$ is $\cF_{t-1}$-measurable. 
Assume $\bLambda_0$ is a $d\times d$ positive definite matrix, and let $\bLambda_t = \bLambda_0 + \sum_{s=1}^t \bphi_s \bphi_s^\top$. 
Then, for any $\delta >0$, with probability at least $1-\delta$, for all $t > 0$, 
\begin{align*}
     \bigg\|\sum_{s=1}^t \bphi_s \eta_s \bigg\|^2_{\bLambda_t^{-1}} \leq 2 R^2 \log \left( \frac{ \det(\bLambda_t)^{1/2} \det(\bLambda_0)^{-1/2} }{\delta}\right). 
\end{align*}
\end{theorem}

\begin{lemma}[Lemma D.4 in \citealt{jin2020provably}]\label{lem: self normalized for V function class raw}
    Let $\cV$ be a function class such that any $V\in \cV$ maps from $\cS \to \RR$ and $\|V\|_{\infty} \leq R$.
    Let $\{\cF_t\}_{t=0}^\infty$ be a filtration. Let $\{s_t\}_{t=1}^\infty$ be a stochastic process in the space $\cS$ such that $s_t$ is $\cF_{t}$-measurable. Let $\{\bphi\}_{t=0}^\infty$ be an $\RR^d$-valued stochastic process such that $\bphi_t$ is $\cF_{t-1}$-measurable and $\|\bphi\|_2 \leq 1$. Let $\bLambda_k = \lambda \bI + \sum_{t=1}^{k-1} \xb_t \bphi_t^\top$. Then for any $\delta>0$, with probability at least $1-\delta$, for any $k$, and any $V \in \cV$, we have 
    \begin{align*}
        \bigg\| \sum_{t=1}^{k-1} \bphi_t \left[ V(s_t) - \EE\left[V(s_t) \mid \cF_{t-1} \right]\right] \bigg\|^2_{(\bLambda_k)^{-1}} & \leq 4 R^2 \left[ \frac{d}{2}\log\left(\frac{k+\lambda}{\lambda} \right) + \log \frac{\cN^{\cV}_\epsilon}{\delta} \right] + \frac{8 k^2 \epsilon^2}{\lambda},
    \end{align*}where $\cN^{\cV}_\epsilon$ is the $\epsilon$-covering number of $\cV$ with respect to the $\ell_\infty$ distance.
\end{lemma}

\end{document}